%% file: root.tex
\documentclass[journal]{IEEEtran}

\input{packages}
\input{commands}

\title{Higher-Order Kinematic Derivatives for Aspect-Preserving Real-Time Numerical Inverse Kinematics}

\author{Steffan~Lloyd%
    \thanks{This work has been submitted to the IEEE for possible publication. Copyright may be transferred without notice, after which this version may no longer be accessible.}%
    \thanks{S. Lloyd is with the Department of Forest and Forest Resources, Norwegian Institute of Bioeconomy Research, H{\o}gskoleveien 8, 1433 {\AA}s, Norway (e-mail: steffan.lloyd@nibio.no).}%
}
\newtheorem{proposition}{Proposition}

\begin{document}

\maketitle

\begin{abstract}
This paper leverages the higher-order kinematic derivatives of serial robots to achieve efficient, aspect-preserving real-time inverse kinematics. First, a family of higher-order solvers is derived, with a necessary orientation-error correction map, yielding the QuIK4 and QuIK5 methods with fourth- and fifth-order convergence. Second, an aspect-preserving step limiter extrapolates the Jacobian determinant along the step to predict and avoid singularities, shortening steps only when they head toward a singularity rather than when the current iterate is ill-conditioned. Third, a Hermite-interpolation certifier verifies the aspect of each answer, and a certified restart algorithm recovers failures from proven same-aspect seeds. The proposed methods are shown to outperform prior methods, reducing failures by up to four orders of magnitude and nearly eliminating aspect switching at tracking-level errors. A C++ library is released with ROS~2 integration.
\end{abstract}


\section{Introduction}

\IEEEPARstart{N}{umerical} inverse kinematics methods are a convenient alternative to analytic solutions when a closed form does not exist or is difficult to derive. They are particularly useful on kinematically calibrated robots~\cite{lloydImprovedAccuracy2024}, where imperfect joint alignment precludes a closed-form solution~\cite{liInverseKinematics2022}, or on manipulators whose geometry was driven by other design criteria, such as avoiding the pinch points common to spherical wrists~\cite{Campeau-Lecours2017}. They also allow straightforward null-space optimization on redundant chains~\cite{Zhao1994}, require no derivation step, which suits computer graphics~\cite{Aristidou2018} and reconfigurable robots~\cite{Chen1999}, and lower the barrier to entry for non-experts, for whom an analytic derivation may be out of reach.

In real-time applications, numerical methods have historically been avoided: they were too slow, could fail to converge, and behaved unpredictably near singularities. The most common approach is the Newton--Raphson (NR) method, which approximates the kinematics linearly about the current guess $\bq_k \in \R^n$ through the geometric Jacobian $\bJ(\bq) \in \R^{6\times n}$ and iterates toward the root as
\begin{align} \label{eq:nr}
    \bq_{k+1} &= \bq_k + \dq_{\rm nr}, &
    \dq_{\rm nr} &=  - \bJ \setminus \b{e},
\end{align}
where $\b{e} = \begin{bmatrix}\b{v}_e^\transpose & \bomega_e^\transpose\end{bmatrix}^\transpose$ is the spatial twist from the current pose to the desired one, and $\b{x} = \b{A} \setminus \b{b}$ denotes the linear solution of $\b{A}\,\b{x} = \b{b}$; for an underdetermined system, the minimum-norm solution is implied.
Existing generic kinematic libraries implement this algorithm: KDL, one of the more popular generic inverse kinematics solvers, offers a Newton--Raphson solver and its damped Levenberg--Marquardt implementation~\cite{smitsKDLKinematics,bruyninckxOpenRobot2001}, and TRAC-IK runs a Newton--Raphson solver simultaneously with a sequential quadratic programming solver~\cite{beesonTRACIK2015}.

Recent work, together with gains in processor speed, has removed these objections. Lloyd \etal~\cite{lloydFastRobust2022} showed that higher-order methods leveraging the kinematic Hessian $\H \triangleq \tfrac{\partial}{\partial \bq} \bJ$ can be applied to inverse kinematics, yielding the step
\begin{align} \label{eq:quik}
    \dq_{\rm quik} = - \left( \bJ + \tfrac{1}{2} \, \H \, [\,\dq_{\rm nr}\,] \right) \setminus \b{e},
\end{align}
where the bracket notation $\tensor{A}\,[\,\b{x}^{i}\,]$, used throughout this paper, denotes the contraction of the tensor $\tensor{A}$ against $i$ copies of the vector $\b{x}$, with the exponent omitted when $i=1$. The resulting \textit{Quick Inverse Kinematics} (QuIK) method converges cubically, runs faster than Newton--Raphson, and improves reliability by more than an order of magnitude. Lloyd {\etal} have since demonstrated this method for real-time inverse kinematics on a kinematically calibrated robotic machining system under admittance control~\cite{lloydPrecisionRobotic2024}. Chen \etal~\cite{chenFifthOrderPOE2026} built on this work, composing a Halley step with a Newton step whose Jacobian is extrapolated from the Hessian of the first point, achieving fifth-order convergence at the cost of an additional error evaluation and linear solve. This ``two-point'' scheme, which evaluates the error at a second point while reusing the derivatives of the first, builds on existing scalar literature~\cite{traubIterative1964,kouFamilyFifthOrder2007}. As the order of a composite iteration depends on how an iteration is defined, we compare methods by order and cost together.

The Newton--Raphson and QuIK methods solve the first- and second-order Taylor models of the error function, respectively~\cite{traubIterative1964,yamamotoTangentHyperbolas1988}. However, this method family can be extended to any order. With the order-$m$ kinematic derivatives denoted as $\Dm{m}$, such that $\Dm{1} = \bJ$ and $\Dm{2} = \H$, the step
\begin{align} \label{eq:taylorstep}
    \dq_{(p)} &= - \Aop{p} \setminus \b{e}, \nonumber\\
    \Aop{p} &= \sum\nolimits_{m=1}^{p} \tfrac{1}{m!} \,\, \Dm{m} \,[\,{\dq_{(p-1)}}\!^{m-1}\,]\,,
\end{align}
solves the degree-$p$ Taylor model of the error, using the order-$(p-1)$ step substituted into its nonlinear term, with $\dq_{(1)} = \dq_{\rm nr}$ and \mbox{$\dq_{(2)} = \dq_{\rm quik}$}. So long as $\Dm{m}$ is exact, the iteration converges at order $p+1$ (\cref{prop:order}, Appendix), and for $p\leq2$ it reproduces the Newton--Raphson and Halley (QuIK) methods.

In practice, higher-order models are rarely used, since $\Dm{m}$ carries $6\,n^m$ partial terms for an $n$-DOF problem, and their cost outweighs the accuracy gained. Serial-chain inverse kinematics is an unusually favorable case: $n$ is small, evaluating $\b{e}$ is relatively expensive, and $\Dm{m}$ follow a Lie-bracket recursion on the joint screws~\cite{mullerHigherDerivatives2014,mullerHigherOrderKinematics2019}, leaving them sparse, symmetric, and inexpensive to compute~\cite{lloydFastRobust2022}. Moreover, \cref{eq:taylorstep} requires $\Dm{m}$ only in contracted form, never as the full tensor, which this paper exploits to accelerate its evaluation.

Previous work on numerical inverse kinematics, \cite{lloydFastRobust2022} included, measures convergence to \textit{any} solution within tolerance as a success~\cite{beesonTRACIK2015,chenFifthOrderPOE2026,lloydFastRobust2022}. A 6R chain can have up to 16 solutions per pose~\cite{raghavanRoth1993}, and on most industrial arms, moving from the seed to a solution in another aspect takes the arm through a singular configuration, where a continuous tool path would require unbounded joint rates. Likewise, solve times are reported as means or medians over large sample sets, but real-time controllers have a fixed period, so a slow solve can be as detrimental as a failed one. Real-time use therefore imposes two constraints that existing benchmarks do not capture:
\begin{enumerate}
    \item the solution must lie in the same aspect as the seed, so that they are joined by a singularity-free path; and
    \item the solver must complete within a fixed deadline, and an overrun is a failure rather than a slow success.
\end{enumerate}
We require that the seed's aspect be maintained, rather than its branch: a solver can readily check the former at every call, from the sign of $\det\bJ$ along the chord $\bq_0 \to \bq^\star$ (\cref{sec:aspect_certify}). On non-cuspidal arms, which include most industrial designs, each aspect holds at most one nonsingular solution per pose, and the two requirements coincide~\cite{wengerChablatReview2023}. On cuspidal arms, such as the Kinova Jaco, several solutions can share an aspect, and a change between them is not immediately detectable from the sign of $\det\bJ$~\cite{salunkheTrajectory2023}. Similarly, kinematic calibration of an industrial arm can make a nominally non-cuspidal arm cuspidal.
In practice, however, this distinction is less important than it may appear. The corridors connecting two same-aspect solutions are often narrow, particularly when the cuspidality arises from a perturbation (e.g., a calibration) of a nominally non-cuspidal model, and from a given seed, typically only one of the solutions can be reached along a straight singularity-free path. If both can be so reached, we consider both equally valid for the purpose of this study.

Existing work has developed strategies to make iterative solvers safer in two ways. Damped least-squares methods shorten and rotate the step when $\bJ$ is ill-conditioned~\cite{nakamuraInverseKinematic1986,wamplerManipulatorInverse1986}. \textit{Adaptive damping} raises the damping factor as the smallest singular value of $\bJ$ falls~\cite{chiaveriniReviewDamped1994}, \textit{selectively damped least squares} (SDLS) damps each singular component of the step separately~\cite{bussSelectivelyDamped2005}, and \textit{Levenberg--Marquardt} variants set the damping from the error norm~\cite{sugiharaSolvabilityUnconcerned2011}. Line searches and trust regions accept a step only when the error decreases~\cite{nocedalNumerical2006}, and step clamps bound its length~\cite{bussSelectivelyDamped2005}. However, these methods judge the step from $\bJ$ at the current iterate, and line searches from the error at the trial point. Neither models how $\bJ$, and with it $\det\bJ$, changes along the step. A step from a well-conditioned configuration is therefore taken in full even if it crosses a singular surface, and a step that does cross one can still reduce the error and pass the line search. Conversely, a step near a singularity is damped, even when it points away from it. The aspect literature has investigated whether configurations are separated by a singularity, but offline and over the whole workspace, rather than under the real-time constraints of a solver~\cite{wengerChablatReview2023,salunkheKinematicIssues2025}.

Higher-order methods such as QuIK gain speed by taking fewer, larger steps; however, a larger step is more likely to cross a singularity and arrive in a different aspect. Damping and step clamps reduce this risk by shortening and warping the step, which discards the accuracy that motivated the higher-order model.
This paper instead uses the higher-order derivatives in three complementary ways for efficient, aspect-preserving real-time inverse kinematics: first, to take more accurate steps through higher-order models; second, to predict where $\det\bJ$ changes sign along a step; and third, to certify the aspect of the final answer. This paper makes five contributions:
\begin{enumerate}
    \item We show that the steps of \cref{eq:taylorstep} require the higher-order kinematic derivatives only as contractions along the step. Evaluated with the Lie-bracket recursion of \cite{mullerHigherOrderKinematics2019} over cumulative joint rates, their cost scales as $O(m^2 n)$ instead of $O(n^m)$ for the full tensors, and the contractions can be evaluated highly efficiently.
    \item We derive the orientation-error correction map required in Taylor models of third and higher order, and through it obtain the \textit{QuIK4} and \textit{QuIK5} methods, which converge at fourth and fifth order.
    \item We propose a step limiter that uses the same kinematic derivatives to extrapolate where $\det\bJ$ changes sign along the step, and shortens the step accordingly to avoid it.
    \item We propose an efficient aspect-certification algorithm that interpolates $\det\bJ$ along the chord from seed to answer with Hermite polynomials, built from the limiter's determinant derivatives, and subdivides the chord until it is cleared or a singularity is found. \textit{Certified restarts} then use this certification to verify answers and, when a solve fails or is refused, restart from a same-aspect seed.
    \item We evaluate these solvers and safeguards against existing methods by the rate at which they fail to return a same-aspect answer within a hard deadline, and release the C++ library with ROS~2 integration.
\end{enumerate}

This paper is organized as follows:
\cref{sec:problem} states the problem and the metrics used throughout.
\cref{sec:derivs} derives the efficient contraction-ladder evaluation of the higher-order kinematic derivatives along a step.
\cref{sec:solvers} introduces the QuIK4 and QuIK5 methods, including the orientation-error correction map.
\cref{sec:limiter} develops the aspect-preserving step limiter from a determinant-derivative front model and Hermite sign-change detection, and 
\cref{sec:aspect_certify} presents the Hermite trust-band certifier and the certified-restart algorithm.
\cref{sec:redudant} discusses the extension to redundant arms via the extended Jacobian.
\cref{sec:validation} evaluates the methods experimentally against prior solvers, and 
\cref{sec:conclusion} concludes.

\section{Problem Statement}\label{sec:problem}

Let $\bT(\bq) \in S\hspace{-.02em}E(3)$ denote the forward kinematics of an $n$-DOF serial chain and let $\bT_d$ be a target pose. The inverse kinematics problem seeks a root $\bq^\star$ such that $\b{e}(\bq^\star) = \b{0}$, where the pose error
\begin{align} \label{eq:error}
    \b{e}(\bq) = \begin{bmatrix} \b{v}_e \\ \bomega_e \end{bmatrix}
    = \begin{bmatrix} \b{p}_n(\bq) - \b{p}_d \\
      \log\!\big(\bR_n(\bq)\,\bR_d^\transpose\big)^{\vee} \end{bmatrix},
\end{align}
where $\log(\cdot)^\vee$ extracts the rotation vector of a rotation matrix, and $\b{e}$ is thus the spatial twist between the current and desired poses, expressed in exponential coordinates as in \cite{lloydFastRobust2022}. With this definition, the geometric Jacobian approximates the error Jacobian to first order, and equals it at a root~\cite{sugiharaSolvabilityUnconcerned2011}, and a solve is declared converged once $\|\b{e}\| < \varepsilon$. We focus on non-redundant 6-DOF chains ($n=6$); however, the solver and derivative formulas in \cref{sec:derivs,sec:solvers} hold for any~$n$. The step limiter and aspect certifier of \cref{sec:limiter,sec:aspect_certify} require $\det\bJ$ and therefore a square Jacobian; they are developed for $n = 6$. \Cref{sec:redudant} discusses the extension of these methods to redundant arms explicitly. 

A configuration is singular when $\det\bJ(\bq) = 0$. The \textit{aspects} of a manipulator are the maximal connected regions of joint space on which $\det\bJ \neq 0$, bounded by singularities~\cite{wengerChablatReview2023}. Two configurations in one aspect are joined by a path on which $\bJ$ stays invertible; every path between aspects meets a singularity. We require that the solution $\bq^\star$ lie in the same aspect as the seed~$\bq_0$, and in testing score a solve as correct only when it does. On the KR6, this is decided exactly from the closed-form factors of $\det\bJ$; chains without such a label are treated in \cref{sec:generality}.

For real-time use, we define a solver deadline $\tau$ set by the control period. A solve can end in one of four ways: it converges to a same-aspect solution; it fails to converge within~$\tau$; the certifier of \cref{sec:aspect_certify} refuses the answer; or it converges to a solution in a different aspect. A \textit{failure} is any outcome other than a same-aspect answer by the deadline. Of these, an unconverged or refused solve is recoverable, since the controller can hold its previous setpoint and retry, but a wrong-aspect answer could command a discontinuity and is far more damaging.

Unless otherwise specified, the primary test chain is a KUKA KR6~\cite{lloydFastRobust2022}, a non-cuspidal, Puma-style robot whose $\det\bJ$ factors in closed form as a product of elbow, shoulder, and wrist terms. The signs of these factors serve as the ground-truth aspect label, independently of the certifier, and are used to score the algorithms. Sample problems are generated as in~\cite{lloydFastRobust2022}: a ground-truth configuration $\bq^\star$ is drawn uniformly from $[-\pi, \pi]^n$ and used to compute the target pose $\bT_d = \bT(\bq^\star)$. The initial guess $\bq_0$ is obtained via a perturbation
\begin{align} \label{eq:sampling}
    \b{r} &\sim \mathcal{U}(-1,1)^n, &
    \bq_0 &= \bq^\star + \sigma \left( \tfrac{1}{n} \|\b{r}\|_1 \right)^{-1} \b{r},
\end{align}
so that $\sigma$ is the mean absolute initial joint error and sets the difficulty of the problem. Problems in which $\bq_0$ lies in a different aspect from $\bq^\star$ are discarded and regenerated, so that $\bq^\star$ is the unique same-aspect answer. Results are reported against $\sigma$ and $\tau$ unless otherwise noted. All timings are single-threaded C++, compiled with {\small \texttt{-O3}} optimization, from the released library, measured on a MacBook Pro with M5~Pro chipset.

\section{Efficient Higher-Order Kinematic Derivatives Along a Step}\label{sec:derivs}

The kinematic Hessian $\H$ can be built from the columns of $\bJ$ alone, with minimal additional computation \cite{hourtashKinematicHessian2005,lloydFastRobust2022}. In screw theory, this is the classical result that the derivative of a joint screw with respect to a preceding joint variable is the Lie bracket of the two screws \cite{mullerHigherDerivatives2014,mullerHigherOrderKinematics2019}. Define $\b{j}_{vk}$ and $\b{j}_{\omega k}$ as the linear and angular columns of $\bJ$, and $\b{p}_n$ as the tool position, so that $\b{j}_{vk} = \partial \b{p}_n/\partial q_k$. Turning a revolute joint $a$ rotates every Jacobian column above it rigidly about that joint's axis, giving
\begin{align} \label{eq:primitives}
    \tfrac{\partial \b{j}_{\omega b}}{\partial q_a} &= \b{j}_{\omega a} \times \b{j}_{\omega b}, &
    \tfrac{\partial \b{j}_{v b}}{\partial q_a} &= \b{j}_{\omega a} \times \b{j}_{v b},
        && a \leq b, \nonumber \\
    \tfrac{\partial \b{j}_{\omega b}}{\partial q_a} &= \b{0}, &
    \tfrac{\partial \b{j}_{v b}}{\partial q_a} &= \b{j}_{\omega b} \times \b{j}_{v a},
        && a > b,
\end{align}
where the two linear cases are related by the symmetry of $\nabla^2 \b{p}_n$, so only one needs to be computed.

The identities behind \cref{eq:primitives}, however, are not specific to obtaining the second derivative from the first. Generalizing, we write $\Dm{m} \triangleq \nabla^{m-1} \bJ$ for the order-$m$ kinematic derivative, with $\Dm{1} = \bJ$ and $\Dm{2} = \H$ as before. Differentiating \cref{eq:primitives} further, the Jacobi identity collapses the two terms of each product rule into one, so every entry of $\Dm{m}$ is a single nested cross product of Jacobian columns \cite{mullerHigherDerivatives2014,mullerHigherOrderKinematics2019}. For joints $a_1 \leq \dotsm \leq a_m$,
\begin{align} \label{eq:chain}
    \hspace{-.2em}\tfrac{\partial^{m-1} }{\partial q_{a_1} \cdots \partial q_{a_{m-1}}} \, \b{j}_{v a_m}
      &= \b{j}_{\omega a_1} \times \Big( \dotsm\, \times
         \big( \b{j}_{\omega a_{m-1}} \times \b{j}_{v a_m} \big) \Big), \nonumber \\
    \hspace{-.2em}\tfrac{\partial^{m-1} }{\partial q_{a_1} \cdots \partial q_{a_{m-1}}} \, \b{j}_{\omega a_m}
      &= \b{j}_{\omega a_1} \times \Big( \dotsm\, \times
         \big( \b{j}_{\omega a_{m-1}} \times \b{j}_{\omega a_m} \big) \Big).
\end{align}
The linear block $\nabla^m \b{p}_n$ is symmetric in all $m$ indices, while an angular entry is zero unless the differentiated column is that of the highest joint. A prismatic joint enters only as the innermost column.

Although each partial derivative term is cheap, there are $6\,n^m$ of them. Critically, however, the full tensors need never be formed: \cref{eq:taylorstep} requires only the contraction $\Dm{m}\dqp{m-1}$, a $6 \times n$ matrix
\begin{align} \label{eq:contraction}
    \Dm{m}\dqp{m-1} = \begin{bmatrix}
        \b{d}^{(m)}_{v1} & \dotsm & \b{d}^{(m)}_{vn} \\[.4em]
        \b{d}^{(m)}_{\omega 1} & \dotsm & \b{d}^{(m)}_{\omega n}
    \end{bmatrix}\!,
\end{align}
whose $m=1$ case is $\bJ$ itself. This matrix is the $(m{-}1)${\th} derivative of $\bJ$ along the ray $\bq + t\,\dq$; equivalently, it is the $(m{-}1)${\th} time derivative of $\bJ$ at the constant joint rate $\dot{\bq} = \dq$ \cite{mullerHigherOrderKinematics2019,fuAnalyticalExpressions2020}. Contracting \cref{eq:primitives} against $\dq$ and pulling the sums inside the cross products, each column needs only the cumulative vectors
\begin{align} \label{eq:cumsums}
    \bomega_k &= \sum\nolimits_{i \leq k} \delta q_i \, \b{j}_{\omega i}, &
    \b{v}_k &= \sum\nolimits_{i > k} \delta q_i \, \b{j}_{vi},
\end{align}
where $\bomega_k$ is the angular velocity of link $k$ when the joints move at rate $\dq$, and $\b{v}_k$ is the contribution to the tool's linear velocity from the joints above $k$. The vector $\dq$ is fixed, so differentiating a cumulative vector returns the same sum raised one order,
\begin{align} \label{eq:cumsums_m}
    \bomega^{(m)}_k &= \sum\nolimits_{i \leq k} \delta q_i \, \b{d}^{(m)}_{\omega i}, &
    \b{v}^{(m)}_k &= \sum\nolimits_{i > k} \delta q_i \, \b{d}^{(m)}_{vi},
\end{align}
with $\bomega^{(1)}_k = \bomega_k$ and $\b{v}^{(1)}_k = \b{v}_k$. Repeated differentiation follows Leibniz's rule and gives the ladder formula
\begin{align} \label{eq:leibniz}
    \b{d}^{(m+1)}_{\omega k}\! &= \!\sum_{s=1}^{m} \tbinom{m-1}{s-1}
        \ \bomega^{(s)}_k \times \b{d}^{(m-s+1)}_{\omega k}, \\
    \b{d}^{(m+1)}_{vk}\! &= \!\sum_{s=1}^{m} \tbinom{m-1}{s-1} \Big(
        \bomega^{(s)}_k \times \b{d}^{(m-s+1)}_{vk}\!
        + \b{d}^{(s)}_{\omega k} \times \b{v}^{(m-s+1)}_k \Big). \nonumber
\end{align}
The first rung, $\b{d}^{(2)}_{\omega k} = \bomega_k \times \b{j}_{\omega k}$, is the form commonly used to compute $\dot{\bJ}$. Here, \cref{eq:leibniz} is the known recursion for higher time derivatives \cite{mullerHigherOrderKinematics2019,fuAnalyticalExpressions2020}, specialized to a constant joint rate and written for the geometric Jacobian, whose tool-point reference adds the second term in the linear row.

Each rung adds $m$ terms per column, so the ladder costs $O(m^2 n)$, instead of $O(n^m)$ for forming the full tensor and contracting it. Each rung also reuses the sums of the rungs below it, so $\Dm{2}\dqp{1}$, $\Dm{3}\dqp{2}$, and $\Dm{4}\dqp{3}$ come from a single sweep over the chain. \Cref{tab:derivcost} gives operation counts and times for $n=6$: through $m=4$, the ladder adds \qty{0.106}{\micro\second} to the \qty{0.057}{\micro\second} of the forward kinematics and Jacobian, against \qty{3.09}{\micro\second} for the full tensors.

\begin{table}[t]
	\caption{Evaluation cost of kinematic derivative contractions for $n=6$.}
	\label{tab:derivcost}
	\footnotesize
    \renewcommand{\arraystretch}{1.3}
	\setlength{\tabcolsep}{3pt}
	\begin{tabularx}{\columnwidth}{
        @{}X@{}
        r@{\hspace{.8em}}
        r@{}
        r@{\hspace{1.3em}} %
        r@{\hspace{.8em}}
        r@{}
    }
		\multicolumn{3}{r@{}}{\textbf{Full tensor \& contraction}}
		& \multicolumn{3}{r@{}}{\textbf{Contraction ladder}} \\
		&$+ \,/\, \times$ & Time [\si{\micro\second}]
		&& $+ \,/\, \times$ & Time [\si{\micro\second}] \\
		\toprule
        \input figures/tab_derivcost
		\bottomrule\\
	\end{tabularx}
    \\[-1em]
    \footnotesize Note: all numbers are non-cumulative, assuming that the lower-rung values have already been computed; the total shows the combined cost for all orders together, with FK and Jacobian. Computation times are averaged over $N=10^6$ samples. Operations and time scale linearly with $n$ for the contraction ladder, whereas they scale as $O(n^m)$ for the full tensor.

\end{table}

The derivatives above describe the forward kinematics of the chain, whereas \cref{eq:taylorstep} is a Taylor model of the error \cref{eq:error}. The two agree only in part: for position the derivative is exact, but for orientation it is not.
Were $\bJ_\omega$ the gradient of some orientation coordinate (i.e., an analytic Jacobian), its mixed partial derivatives would commute, and $\partial \b{j}_{\omega k}/\partial q_i$ would equal $\partial \b{j}_{\omega i}/\partial q_k$ for every $i < k$. Instead, $\bJ_\omega$ describes the relationship $\bomega = \bJ_\omega \, \dot{\bq}$, and by \cref{eq:primitives} we have $\partial \b{j}_{\omega k}/\partial q_i = \b{j}_{\omega i} \times \b{j}_{\omega k}$ while $\partial \b{j}_{\omega i}/\partial q_k = \b{0}$: the mixed partials differ whenever the two axes are not parallel, so no orientation coordinate has $\bJ_\omega$ as its gradient. The angular blocks of $\Dm{m}$ are thus not the Taylor coefficients of the orientation error. This discrepancy cancels in the first- and second-order models of \cref{eq:nr,eq:quik}; for higher-order models, it does not. \Cref{sec:solvers} derives the mismatch precisely, along with the correction terms needed.

\section{The QuIK4 and QuIK5 Methods}\label{sec:solvers}

Newton--Raphson and QuIK are the first two members of \cref{eq:taylorstep}, and coincide with the order-1 and order-2 Householder methods (Newton and Halley). This section defines two further members: \textit{QuIK4} ($p=3$) and \textit{QuIK5} ($p=4$), each an application of \cref{eq:taylorstep} seeded with the step estimate from the preceding member.

With $\dq_{\rm nr} = \dq_{(1)}$ and $\dq_{\rm quik} = \dq_{(2)}$ as in \cref{eq:nr,eq:quik}, the nominal QuIK4 step $\dq_{(3)}$ is
\begin{align}\label{eq:quik4}
    \!\!\dq_{(3)} &= -\b{A}_3 \setminus \b{e}, \nonumber\\
    \b{A}_3 &= \bJ + \tfrac{1}{2} \H\,[\,\dq_{(2)}\,] + \tfrac{1}{6} \Dm{3}[\,\dq_{(2)}^{\,2}\,].
\end{align}
The QuIK5 step $\dq_{(4)}$ extends the sequence one rung further. Seeded with the QuIK4 step $\dq_{(3)}$, its nominal equation is
\begin{align}\label{eq:quik5}
    \!\!\dq_{(4)} &= -\b{A}_4 \setminus \b{e},\\
    \b{A}_4 &= \bJ + \tfrac{1}{2} \, \H\,[\,\dq_{(3)}\,]
               + \tfrac{1}{6} \, \Dm{3}[\,\dq_{(3)}^{\,2}\,]
               + \tfrac{1}{24} \, \Dm{4}[\,\dq_{(3)}^{\,3}\,]. \nonumber
\end{align}
QuIK4 and QuIK5 are named after the convergence order they achieve, not their model order $p$: QuIK4 converges at fourth order and QuIK5 at fifth, since each order-$p$ member converges at order $p+1$ (\cref{prop:order}, Appendix).

Implemented naively as above, however, neither method reaches its full rate; both cap at third-order convergence, the same as baseline QuIK. The problem is the orientation error. The linear part of $\b{e}$ in \cref{eq:error} is a point difference whose derivatives from \cref{sec:derivs} are exact. The Jacobian and its derivatives, however, describe the \textit{angular velocity} of the tool, not the orientation error. Unlike positions, rotations compose rather than add: the composition of two rotation vectors $\b{a}$ and $\b{b}$ yields $\b{a} + \b{b} + \tfrac{1}{2} \b{a} \times \b{b} + \dotsm$, not $\b{a} + \b{b}$. Nor can the issue be sidestepped by redefining the error, since no orientation coordinate has $\bJ_\omega$ as its gradient (\cref{sec:derivs}).

Consider the step along $\bq + t\,\dq$ for $t \in [0,1]$, and let $\bpsi(t)$ be the rotation vector of the tool's own motion up to $t$,
\begin{align} \label{eq:psi_def}
    \bR(\bq + t\,\dq) = \exp\!\big(\skewMat{\bpsi(t)}\big) \, \bR(\bq).
\end{align}
The rotation $\bpsi(t)$ is determined by the chain and step alone, not the target. By construction $\bpsi(0) = \b{0}$, and at the full step we require $\exp(\skewMat{\bpsi(1)}) = \bR_d \, {\bR_n}\!\!^\transpose$, so
\begin{align} \label{eq:psi_target}
    \bpsi(1) = \log\!\big(\bR_d \, {\bR_n}\!\!^\transpose\big)^{\vee} = -\bomega_e,
\end{align}
where $\bomega_e$ is the angular block of $\b{e}$.

\begin{figure*}[b]
    \centering
	\includegraphics[width=.99\textwidth, trim={0.4cm 0.2cm 0.cm 0.1cm},clip]{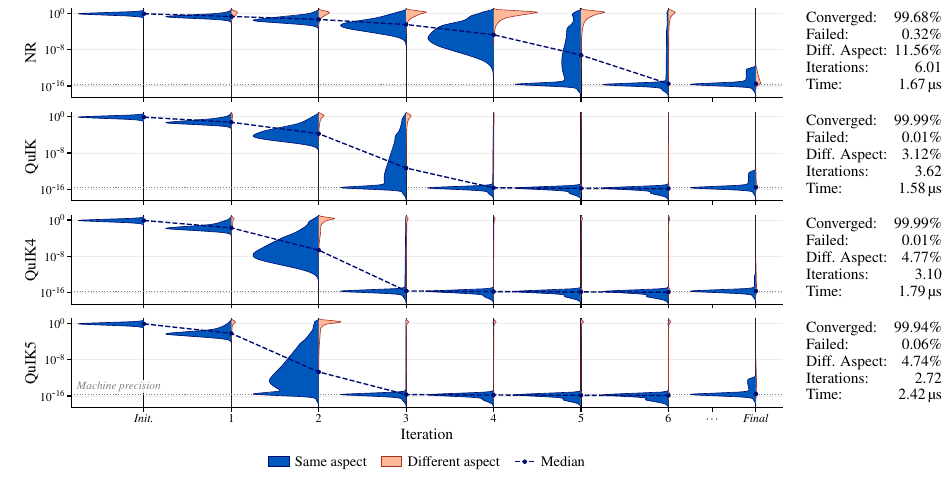}
    \caption{Convergence of the error $\|\b{e}\|$ on the KUKA KR6 with mean absolute initial joint error $\sigma=0.3$\,rad, over $N=10^6$ random poses. Right-hand statistics give the percentage converged (to any solution), the percentage failed, the percentage converged to a different aspect, the interpolated median number of iterations, and the median solve time. Distributions are scaled individually to show shape; they are not to scale relative to each other.}
	\label{fig:convergence}
\end{figure*}

To solve for $\dq$, we expand $\bpsi(t)$ about $t=0$. Along the ray, the tool turns at $\bomega_n(t) = \bJ_\omega(\bq + t\,\dq)\,\dq$, whose derivatives in $t$ are the same contractions of \cref{sec:derivs},
\begin{align} \label{eq:omega_derivs}
    \bomega^{(m)}_n = \tfrac{\partial^{m-1}}{\partial t^{m-1}}
        \, \bomega_n(t) \Big|_{t=0}
        = \big(\Dm{m}\dqp{m-1}\big)_\omega \, \dq ,
\end{align}
so that $\bomega^{(1)}_n = \bJ_\omega \dq$ and $\bomega^{(2)}_n = (\H\dqp{})_\omega \dq$.
If rotations were additive, $\bpsi(1)$ would be a direct Taylor sum of these terms, exactly as supplied by \cref{eq:quik4,eq:quik5}. Instead, rotations compose, and the accumulated rotation obeys \cite[Sec.~7.2]{barfootStateEstimation2017}
\begin{align}\label{eq:psi_expansion}
    \tfrac{\partial}{\partial t}\, \bpsi = \bomega_n - \tfrac{1}{2} \bpsi \times \bomega_n + \tfrac{1}{12} \bpsi \times (\bpsi \times \bomega_n) - \dotsm.
\end{align}
Equation \cref{eq:psi_expansion} can be used directly to derive the Taylor coefficients $\bpsi_k = \bpsi^{(k)}(0)$, giving the expansion
\begin{align}\label{eq:psi_exp}
    \bpsi(t)
    &= \bpsi_0 + \bpsi_1 \, t + \tfrac{1}{2}\, \bpsi_2 \, t^2 + \tfrac{1}{6}\, \bpsi_3 \, t^3 + \tfrac{1}{24} \, \bpsi_4 \, t^4 + \dotsm,
\end{align}
where $\bpsi_0 = \b{0}$ and
\begin{align*}
    \bpsi_1 &= \bomega_n^{(1)},
    & 
    \bpsi_2 &= \bomega_n^{(2)},
    \nonumber\\
    \bpsi_3 &= \bomega_n^{(3)} + \tfrac{1}{2} \bomega_n^{(2)} \times \bomega_n^{(1)},
    & 
    \bpsi_4 &= \bomega_n^{(4)} + \bomega_n^{(3)} \times \bomega_n^{(1)}.
\end{align*}
Then, evaluating at $t=1$ through fourth order, we obtain
\begin{align}
    \bpsi(1) &= 
    \bomega_n^{(1)} + \tfrac{1}{2} \,  \bomega_n^{(2)} + \tfrac{1}{6} \,  \bomega_n^{(3)}+ \tfrac{1}{24} \, \bomega_n^{(4)} 
        \\
    &\phantom{=} \ + \tfrac{1}{12}\, \bomega_n^{(2)} \times \bomega_n^{(1)} + \tfrac{1}{24}\, \bomega_n^{(3)} \times \bomega_n^{(1)} 
        \nonumber\\
    &= \big[\b{A}_p\,\dq\big]_{\omega} + \tfrac{1}{12}\, \bomega_n^{(2)} \times \bomega_n^{(1)} + \tfrac{1}{24}\, \bomega_n^{(3)} \times \bomega_n^{(1)}.\nonumber
\end{align}
The first line contains the standard Taylor terms $\tfrac{1}{m!}\bomega_n^{(m)}$ already present in $[\b{A}_p\,\dq]_{\omega}$, but the second line contributes two cross-product corrections that the kinematic derivatives alone do not capture. These corrections are the same order as the third- and fourth-order terms in $\b{A}_p$, so QuIK4 and QuIK5 cannot reach their full convergence rate without them.
This analysis also explains why Newton--Raphson and QuIK required no corrections: $\bpsi_1 = \bomega_n^{(1)}$ and $\bpsi_2 = \bomega_n^{(2)}$ coincide with the terms already in the model, and the additional cross products first appear in $\bpsi_3$ and $\bpsi_4$.

These terms are carried into the QuIK4 and QuIK5 solvers by modifying \cref{eq:quik4,eq:quik5} as
\begin{align}\label{eq:quik_corr}
    \dq_{(p)} &= - \b{A}_p \setminus \big( \b{e} + \b{c}_p \big) \nonumber\\
    \b{c}_p &= 
    \begin{bmatrix}
        \b{0} \\
        \tfrac{1}{12} \bomega_n^{(2)} \times \bomega_n^{(1)} + \tfrac{1}{24} \, \bomega_n^{(3)} \times \bomega_n^{(1)}
    \end{bmatrix}\!.
\end{align}

At each stage, \cref{eq:leibniz} generates all contracted derivative terms in a single pass, plus two additional cross products for the correction. The contraction ladder must be recomputed at each level of \cref{eq:taylorstep}, since each level contracts along a different step: QuIK contracts along $\dq_{\rm nr}$, QuIK4 along $\dq_{\rm quik}$, and QuIK5 along $\dq_{(3)}$. Even so, contracting along a vector remains far cheaper than forming the full tensors.

One safeguard proved necessary when applying \cref{eq:quik_corr}: a near-singular intermediate solve can occasionally return a large step. The matrix $\b{A}_p$ contains terms contracted from this estimated step, and the contraction grows as $\|\dq\|$ for $\Dm{2}$ but as $\|\dq\|^3$ for $\Dm{4}$. A large step can therefore swamp the lower-order terms in $\b{A}_p$, inflating the matrix and driving the corrected step to near zero. This traps the solver in the singular region. To prevent this, the substituted estimate is used only within a trust radius $\rho = c \, ( \|\b{v}_e\| / L + \|\bomega_e\| )$, where $L$ is a characteristic length of the chain that normalizes linear and angular measurements; if $\|\dq\| > \rho$, the Newton step $\dq_{\rm nr}$ is substituted instead. Setting $c=500$ corrected most occurrences while activating on under 1\% of iterations.

\cref{fig:convergence} runs the four methods on $N = 10^6$ poses of the KUKA KR6 robot, seeded with $\sigma=0.3$\,rad mean absolute joint error, without the joint-step cap, the Cartesian clamp, or the limiters described later in \cref{sec:limiter}. The ladder behaves as derived: after two iterations, the median error is $4.6\cdot10^{-2}$ for Newton--Raphson, $1.9\cdot10^{-4}$ for QuIK, $2.3\cdot10^{-7}$ for QuIK4, and $2.0\cdot10^{-11}$ for QuIK5, and the median iteration count falls from 6.0 for Newton--Raphson to 3.6 (QuIK), 3.1 (QuIK4), and 2.7 (QuIK5). The higher-order methods cost more per iteration, however, and the median solve time is \qty{1.58}{\micro\second} for QuIK, \num{1.79} and \qty{2.42}{\micro\second} for QuIK4 and QuIK5, and \qty{1.67}{\micro\second} for Newton--Raphson.

Importantly, every method regularly converges to a solution in a different aspect: QuIK changed aspects on 3.1\% of solves, against 4.8\% for QuIK4 and 4.7\% for QuIK5. Newton--Raphson was worst at 11.6\%. The following section addresses this issue directly, using the same higher-order derivative information to predict singularities along a step.

\section{Aspect-Preserving Step Limiter}\label{sec:limiter}

The previous section shows that higher-order methods can converge in fewer iterations. However, in doing so, they tend to take longer steps, allowing the iterate to cross a singular surface. The key insight of this section is that the same higher-order derivatives used by \cref{sec:solvers} to locate the desired pose can also be used to locate, and avoid, singularities. The boundaries of an aspect are defined by the singular condition $\det\bJ = 0$, a polynomial in the entries of $\bJ$, whose derivatives \cref{sec:solvers} already forms, with no further kinematics or decompositions. The determinant also passes \textit{through} zero at the boundary, becoming negative on the other side. In this section, we use both properties in complementary ways. First, an inexpensive second-order Taylor model of $\det \bJ$ along the step direction is used as a \textit{front model} to predict the crossing point and to shorten the step accordingly. Second, we use a \textit{Hermite sign-change detection algorithm} to detect crossings the front model missed, and to backtrack accordingly before the step is accepted.

\subsection{A determinant-derivative front model along the step}\label{sec:limiter_model}

Consider again the ray $\bq + t\, \dq$, and let \mbox{$\b{A}(t) \triangleq \bJ(\bq + t\,\dq)$} be the manipulator Jacobian along the ray, and \mbox{$\Delta(t) \triangleq \det\b{A}(t)$} its determinant. Then, a Taylor-series expansion of $\b{A}(t)$ about $t=0$ is
\begin{align}\label{eq:A_exp}
    \b{A}(t) = \bJ + t\b{A}^{(1)}(0) + \tfrac{t^2}{2} \b{A}^{(2)}(0) + \dotsm
\end{align}
where $\b{A}^{(m)}(0)$ are the tensor contractions along the step from \cref{eq:leibniz},
\begin{align}
    \b{A}^{(m)}(0) &= \Dm{m+1}[\dq^m], & m&=1,2.
\end{align}
Define $\b{M}_m \triangleq \bJ^{-1} \, \b{A}^{(m)}(0)$, and \cref{eq:A_exp} becomes
\begin{align}
    \b{A}(t) = \bJ \, \Big( \,\b{I} + t\,\b{M}_1 + \tfrac{t^2}{2}\,\b{M}_2 + \dotsm \,\Big) = \bJ \, \Big(\b{I} + \b{X}\Big),
\end{align}
with $\b{X} = t\,\b{M}_1 + \tfrac{t^2}{2}\,\b{M}_2 + \dotsm$. Determinants are multiplicative, so we can cancel $\Delta(0) = \det\bJ$ as
\begin{align}
    \Delta(t)/\Delta(0) = \det\!\Big( \b{I} + \b{X}\Big).
\end{align}
Let $\lambda_i$ be the eigenvalues of $\b{X}$. Then the eigenvalues of $\b{I}+\b{X}$ are $1+\lambda_i$, and the determinant is the product of eigenvalues
\begin{align}
    \hspace{-.6em}\det\!\Big( \b{I} + \b{X}\Big) \!&= \!\prod\nolimits_i (1+\lambda_i) = 1 + \sum\nolimits_i \!\!\lambda_i + \!\sum\nolimits_{i<j} \!\!\lambda_i \lambda_j + \dotsm \nonumber\\
    &= 1 + \tr\b{X} + \tfrac{1}{2}\Big[ (\tr\b{X})^2 - \tr(\b{X}^2) \Big] + \dotsm,
\end{align}
where the second line simplifies using the identities \mbox{$\tr\b{X} = \sum_i\lambda_i$}, \mbox{$\tr(\b{X}^2) = \sum_i \lambda_i^2$}, and \mbox{$\big(\sum_i \lambda_i\big)^2 = \sum_i \lambda_i^2 + 2\sum_{i<j}\lambda_i\lambda_j$}.
Dropping higher-order terms leaves the quadratic model
\begin{align}\label{eq:detmodel}
    p_2(t) = \Delta(0)\,\big(1 + c_1 t + c_2 t^2\big),
\end{align}
where
\begin{align}
    c_1 &= \tr\,\b{M}_1, &
    c_2 &= \tfrac{1}{2} \big( \tr\b{M}_2 + (\tr\b{M}_1)^2 - \tr(\b{M}_1^{\,2}) \big).
\end{align}
As nothing in this derivation depends on expanding about \mbox{$t=0$}, the determinant derivatives generalize, with \mbox{$\b{M}_m(t) \triangleq \b{A}(t)^{-1}\,\b{A}^{(m)}(t)$}, to
\begin{align}\label{eq:detderiv}
    \hspace{-.3em}\Delta^\prime(t) &= \Delta(t)\, \tr\b{M}_1(t), \nonumber\\
    \hspace{-.3em}\Delta^{\prime\prime}(t) &= \Delta(t) \Big[ \tr\b{M}_2(t) + (\tr\b{M}_1(t))^2 - \tr(\b{M}_1(t)^2) \Big],
\end{align}
and $p_2$ is thus the second-order Taylor polynomial of $\Delta$ about $t=0$, with $\Delta^\prime(0) = \Delta(0)\,c_1$ and $\Delta^{\prime\prime}(0) = 2\,\Delta(0)\,c_2$.
Taylor expansions of the minors of $\bJ$, $\det\bJ$ included, are known to arbitrary order \cite{mullerHigherOrderKinematics2019}, where they serve offline singularity analysis. Here, the same second-order model is used online, as a step limiter.

The front model takes the nearest predicted singularity crossing $t^\star$ to be the smallest positive root of \cref{eq:detmodel}, and limits the step to
\begin{align}\label{eq:alpha_model}
    \hspace{-.6em}\alpha &= \min\big(\alpha_{\rm cap}, \, \beta_{\rm fm} \, t^\star\big), &
    \alpha_{\rm cap} &= \min\big(1, \, \theta_{\rm max} / \|\dq\|\big),
\end{align}
where $\beta_{\rm fm} < 1$ is a safety factor, and $\alpha_{\rm cap}$ is a fixed cap bounding the norm of the joint step to $\theta_{\rm max}$. The cap is independent of the determinant, and guards against a step long enough to cross into an adjacent aspect and back out again. \Cref{fig:detroot} shows a sample step of this method on the determinant surface of $q_2, q_3$ on a 6-DOF robot: the quadratically approximated root $t^\star$ (red) predicts the true root $t_\times$ with some error, and the step $\beta_{\rm fm} t^\star$ is then backed off further to account for model errors. Two positive roots of $p_2$ mean the step dips through a boundary and returns; the smaller root is then used. No positive real root means no singularity is predicted in front of the step, and the step is limited by the cap $\theta_{\rm max}$ alone.

\begin{figure}[t!]
	\includegraphics[width=1\columnwidth, trim={.95cm .35cm 1.05cm .9cm},clip]{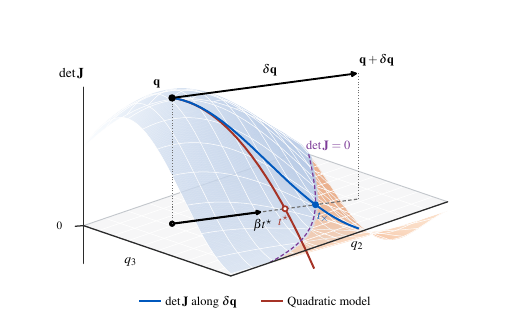}
	\caption{Visualization of the quadratic determinant model on the $q_2$, $q_3$ surface of the KUKA KR6 determinant, near a singularity, for a sample step vector $\dq$. The true singularity crossing $t_\times$ is approximated by the quadratic model at $t^\star$, then the actual step taken is reduced further to $\beta_{\rm fm} t^\star$.}
	\label{fig:detroot}
\end{figure}

\subsection{Hermite sign-change detection and backtracking}\label{sec:limiter_hermite}

The quadratic front model predicts a crossing from derivatives at $t=0$ alone, and its error grows along the step. In this section, the step $\alpha$ from \cref{eq:alpha_model} is verified before it is taken: first by inexpensive checks that clear most steps, and then, when these fail, by a more accurate backtracking model.

First, the step is probed by evaluating the kinematics at $\bq + \alpha\dq$ and computing $\Delta(\alpha)$. If the step is accepted, this probe is free, since its forward kinematics and Jacobian are reused in the next iteration. If $\Delta(\alpha)$ and $\Delta(0)$ differ in sign, a singularity has been crossed and the step must be shortened. If they agree, the test is inconclusive: either no singularity was crossed, or an even number were. To resolve this, let $u = t/\alpha$ be a rescaled step coordinate that places the probe at $u=1$. The value, slope, and curvature of $\Delta$ at $u=0$, along with its value at $u=1$, define a unique cubic Hermite interpolant~\cite{stoerIntroductionNumerical2002}
\begin{align}
    p_3(u) &= a_0 + a_1 u + a_2 u^2 + a_3 u^3,\hspace{-6em}\nonumber\\
    \label{eq:hermite3}
    a_0 &= \Delta(0), &
    a_1 &= \alpha\,\Delta^\prime(0), \nonumber\\
    a_2 &= \tfrac{1}{2}\,\alpha^2\,\Delta^{\prime\prime}(0), &
    a_3 &= \Delta(\alpha) - a_0 - a_1 - a_2,
\end{align}
where $\Delta(0)$, $\Delta^\prime(0)$, and $\Delta^{\prime\prime}(0)$ are already available from the front model. The first three terms of $p_3$ are exactly $p_2(\alpha u)$, and the cubic term warps the front model to pass through the measured endpoint, $p_3(1) = \Delta(\alpha)$. If $p_3$ has no root on $u \in (0, U]$, the step is accepted. We set $U = 1/\beta_{\rm bt}$ on the first probe and $U = 1$ thereafter, so that the first fit extrapolates past the probe and a crossing beyond the step is still shortened.

However, if the signs differ, or if $p_3$ has a root, the interpolant is refined with the slope at the probe, computed from \cref{eq:detderiv} at $t=\alpha$. The five endpoint conditions then define the unique, more accurate quartic Hermite interpolant
\begin{align}\label{eq:hermite4}
    p_4(u) &= a_0 + a_1 u + a_2 u^2 + (4 a_3 - b)\,u^3 + (b - 3 a_3)\,u^4, \nonumber\\
    b &= \alpha\,\Delta^\prime(\alpha) - a_1 - 2 a_2,
\end{align}
with $a_0,\dots,a_3$ as in \cref{eq:hermite3}, such that $p_4(1) = \Delta(\alpha)$ and $p_4^\prime(1) = \alpha\,\Delta^\prime(\alpha)$. If $p_4$ has no root on $(0, U]$, the step is accepted. Otherwise, let $u_k^\times$ be the first root of $p_4$ at probe $k$, and shorten the step to
\begin{align}\label{eq:backoff}
    \alpha_{k+1} = \beta_{\rm bt} \, \min(u_k^\times, 1) \, \alpha_k,
\end{align}
where $0 < \beta_{\rm bt} < 1$ is the backtracking factor, and the clamp prevents the step from lengthening if $U>1$ on the first probe. The process then repeats with a new probe at $\alpha_{k+1}$, and terminates in one of three ways: a probe is accepted and its kinematics are carried into the next iteration; the probe budget is exhausted and the step is taken unverified; or $\alpha$ falls below a floor $\alpha_{\rm min}$ and the solver declines to answer.

\Cref{fig:hermite} shows the full limiter sequence on an unlucky sample step of the QuIK4 algorithm on a solve of the KR6 manipulator. In this example, the front model estimate is poor, and the first probe is placed beyond the singular crossing at $\alpha_1 = \beta_{\rm fm} t^\star$. As $\Delta$ dips through zero and back before $\alpha_1$, the sign of $\Delta(\alpha_1)$ matches $\Delta(0)$. Accordingly, the cubic interpolant $p_3$ is formed, which misses the dip but still fails on $[0, U]$, just past the probe in the look-ahead window. The slope at the probe is then used to refine the fit, forming $p_4$, which correctly locates the dip at $\alpha_1 u_1^\times$. The step backs off to $\alpha_2 = \beta_{\rm bt} u_1^\times \alpha_1$, where the second round forms a cubic (not shown) that has no root, and the step is accepted.

The limiter adds little cost to an iteration. The front model requires one sweep of \cref{eq:leibniz} through to $\Dm{3}\dqp{2}$, one LU factorization, and two solves, costing approximately \qty{0.37}{\micro\second} per step on the KR6 in testing. An accepted probe requires only a further determinant computation, at \qty{0.08}{\micro\second}, and the first probe is accepted on about 99.9\% of steps (with error level $\sigma = \pi/4$ and joint cap $\theta_{\rm max} = 0.5\,\si{rad}$). A rejected probe requires additional computations, costing an additional \qty{0.28}{\micro\second} on average, but these computations are only required when a step would otherwise cross or approach a boundary.

This limiter acts, in effect, as a form of damping, but one that is applied surgically. Damped least-squares methods and their adaptive variants \cite{nakamuraInverseKinematic1986, wamplerManipulatorInverse1986, bussSelectivelyDamped2005, chiaveriniReviewDamped1994} judge proximity to a singularity by the conditioning of $\bJ$ at the current iterate, and shorten or warp the step when it is poor, regardless of the step's direction. The limiter proposed here instead predicts how the determinant changes along the step, and shortens it only when it approaches a boundary: a step from a well-conditioned pose toward a singularity is shortened, while one at an ill-conditioned pose facing away from the singularity is left intact. The step direction is always preserved, and only its length is reduced, by as much as the predicted distance to the boundary requires. The limiter is also faster: adaptive damping variants typically require a singular value decomposition every iteration, on average about twice as expensive (\qty{0.95}{\micro\second} versus \qty{0.45}{\micro\second} on an $n=6$ chain).

\begin{figure}[t!]
    \includegraphics[width=1\columnwidth,trim={0.2cm 0.27cm 0.15cm 0.2cm},clip]{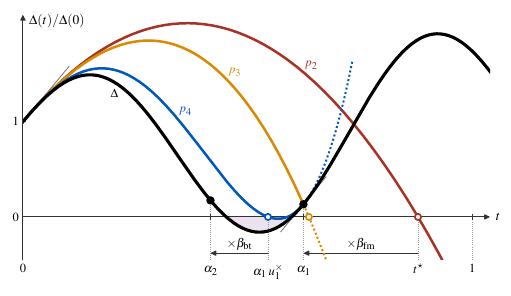}
    \caption{Hermite sign-change detection on an unlucky KUKA KR6 step where the front model prediction is inaccurate. The quadratic estimator $p_2$ limits the first probe to $\alpha_1$, where the cubic $p_3$ fails. The quartic $p_4$ locates the crossing at $\alpha_1 \, u_1^\times$, and the step backs off to the accepted $\alpha_2$.}
    \label{fig:hermite}
\end{figure}

Although accurate, the limiter is designed for speed, and does not certify aspect preservation: a pair of closely spaced crossings between probes can escape the interpolants. The following section addresses the problem of real-time certification more directly.

\section{Aspect Certification and Restarts}\label{sec:aspect_certify}

The limiter of \cref{sec:limiter} judges each step quickly from local data, and a solve that slips past it returns an answer in another aspect that otherwise appears correct. In this section, we instead check the answer \textit{after} the solve. Once a solve converges to $\bq^\star$, the answer can be certified by showing that the chord from the seed $\bq_0$ to $\bq^\star$ never crosses $\det\bJ = 0$. The converse does not hold: $\bq^\star$ can lie in the same aspect as $\bq_0$ while the chord leaves the aspect and re-enters it, in which case certification refuses a correct answer. A refusal can be recovered from, however, and \cref{sec:restarts} introduces \textit{aspect-certified restarts} to recover solves that failed, were refused, or converged to a wrong aspect.

\subsection{Certifying the chord with a Hermite trust band}\label{sec:certify}

Let $\Delta(t) \triangleq \det \bJ\big(\bq_0 + t\,(\bq^\star - \bq_0)\big)$, with $t \in [0,1]$, now denote the determinant along the chord. Its derivatives follow from \cref{eq:detderiv} with $\dq = \bq^\star - \bq_0$. If $\Delta(0)$ and $\Delta(1)$ differ in sign, the chord crosses a singularity and the answer is refused directly. If they agree, then as before the chord either stays nonsingular throughout, or crosses an \textit{even} number of times. This latter case, a \textit{hidden} crossing, is more challenging: nearly all hidden crossings arise where the chord passes close to the intersection of two singular surfaces, and the resulting excursion from the aspect can be arbitrarily narrow, such that fixed sampling density cannot catch them all. Rather, we need a method fast enough to run in real time, but accurate enough that off-aspect answers are not accepted.

We reuse the Hermite interpolation of \cref{sec:limiter_hermite}, but now apply it for accuracy rather than speed, to search $\Delta(t)$ for excursions from the aspect. On an interval $[t_{\rm L}, t_{\rm R}]$ of width $h = t_{\rm R} - t_{\rm L}$, with normalized position $u = (t - t_{\rm L})/h \in [0,1]$, the value, slope, and curvature of $\Delta$ at both ends define a unique quintic Hermite interpolant $p_5(u)$, the two-sided counterpart of $p_3$ and $p_4$,
\begin{align}\label{eq:hermite5}
    p_5(u) &= a_0 + a_1 u + a_2 u^2 + (10 a_3 - 4 b + d)\,u^3 \nonumber\\
    &\phantom{=}\ + (7 b - 15 a_3 - 2 d)\,u^4 + (6 a_3 - 3 b + d)\,u^5,
    \nonumber\\
    d &= \tfrac{1}{2}\,h^2\,\Delta^{\prime\prime}(t_{\rm R}) - a_2,
\end{align}
where $a_0,\dots,a_3$ and $b$ are as in \cref{eq:hermite3,eq:hermite4}, with derivatives taken at $t_{\rm L}$ and $t_{\rm R}$ in place of $0$ and $\alpha$, and scaled by $h$ in place of $\alpha$. The interpolation error is known:~\cite[Sec.~2.1.5]{stoerIntroductionNumerical2002}: for each $u \in [0,1]$ there is some $\xi \in [t_{\rm L}, t_{\rm R}]$ such that
\begin{align}\label{eq:hermite_remainder}
    \big|\Delta(t_{\rm L} + h u) - p_5(u)\big| &= \frac{h^6}{6!} \, u^3 (1-u)^3 \, \big|\Delta^{(6)}(\xi)\big|,
\end{align}
which vanishes at both ends, peaks at the midpoint, and is reduced by a factor of $2^6$ each time the interval is halved. This expression cannot be used directly, since the sixth derivative $\Delta^{(6)}$ is impractical to compute along the chord. Instead, the interpolation error is measured at a single point and used to calibrate \cref{eq:hermite_remainder}, assuming that $\Delta^{(6)}$ is roughly constant over the interval. The determinant is evaluated through the robot kinematics at the minimum $u_{\rm m}$ of $p_5$, or at the midpoint if that minimum lies at an endpoint. A sign change there proves a crossing, and the chord is refused. Otherwise, the residual
\begin{align}\label{eq:remainder}
    e_{\rm m} = |\Delta(t_{\rm L} + h u_{\rm m}) - p_5(u_{\rm m})|
\end{align}
calibrates the error scale of \cref{eq:hermite_remainder}, giving the estimate
\begin{align}
    e(u) &= \eta \, u^3 (1-u)^3,  &
    \eta &= e_{\rm m} / \big(u_{\rm m}(1-u_{\rm m})\big)^3,
\end{align}
so that $e(u_{\rm m}) = e_{\rm m}$. Taking $\Delta > 0$ at the ends without loss of generality, the interval is accepted if
\begin{align}\label{eq:trust}
    p_5(u) - \kappa e(u) &> 0, & u &\in [0, 1],
\end{align}
that is, if $p_5$ stays positive by a margin of at least $\kappa$ times the estimated error, where $\kappa \geq 1$ is a safety factor on the error in the estimate itself. The left side of \cref{eq:trust} is a sextic that is positive at both ends, so the test reduces to checking that it has no real root on $[0,1]$. If no real roots exist, the interval is certified. Otherwise, the interval is split at $u_{\rm m}$, where $\Delta$ is already known, and each part is tested in turn. The sub-intervals are cheaper than the first: both of their endpoints already have a determinant and Jacobian factorization, and at least one endpoint already has its derivatives computed. The process repeats until every interval is accepted, a sign change is found, or the evaluation budget is exhausted.

\begin{figure}[t!]
    \includegraphics[width=1\columnwidth,trim={0.1cm 0cm 0cm 0cm},clip]{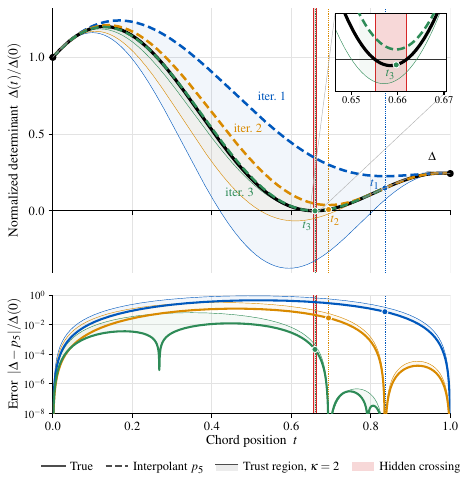}
    \caption{Hermite certification of a KUKA KR6 chord with a hidden crossing $0.007$ wide. Top: $\Delta(t)$ and the quintic interpolant $p_5$ of each iteration (dashed), shaded down to its trust band $p_5 - \kappa\, e(u)$, with $\kappa = 2$. Each iteration evaluates $\Delta$ at the minimum $t_k$ of $p_5$ and splits there; $\Delta(t_3) < 0$ shows the crossing. Bottom: the true interpolation error of each iteration and trust band estimate.}
    \label{fig:certify}
\end{figure}

\cref{fig:certify} shows this procedure on a \qty{3}{rad} chord of the KR6, whose determinant has a hidden crossing only $0.007$ wide and $10^{-4}\Delta(0)$ deep. A uniform grid would require about 150 samples to land reliably inside the excursion; the Hermite certifier locates it in three. The first interpolant misses the dip, with its minimum at $t_1 = 0.84$, but its trust band falls far below zero, so the interval is split. The sign test at $t_1$ passes, and the interpolant on $[0, t_1]$ is an order of magnitude more accurate, but its trust band still does not clear the interval, so it is split again at its minimum $t_2 = 0.69$. The third interpolant, on $[0, t_2]$, lies within about 1\% of the true curve throughout and places its minimum inside the excursion, where the evaluation at $t_3 = 0.66$ fails the sign test and proves a crossing.
We use $\kappa = 2$ here for illustration. In practice, a larger safety factor provides higher accuracy. In this work, we use $\kappa=16$ throughout, which reproduces these three iterations identically.

The trust test is a heuristic, not a proof, as $\eta$ is estimated from a single residual sample. In practice, however, it is highly accurate. \cref{fig:certify_front} compares the proposed certifier against six alternatives on $10^6$ random chords of the KR6 that cross a singular surface and $10^6$ that do not, drawn in equal numbers at lengths of 0.1, 0.3, 1, and \qty{3}{rad}. A certifier can err in two ways: a \textit{miss}, when a crossing chord is accepted, or a \textit{false alarm}, when a non-crossing chord is refused, including when its evaluation budget runs out. \textbf{H5 trust} is the method described above. \textbf{D1} samples the sign of $\Delta$ on a uniform grid of $N$ points. \textbf{D2} bounds $\Delta$ between samples with an adaptive Lipschitz envelope \cite{piyavskiiAlgorithmFinding1972, shubertSequentialMethod1972}, taking the slope bound as a safety factor times the larger slope at the interval's ends. \textbf{C1} and \textbf{C2} mirror \textbf{D1} and \textbf{D2}, but act on the smallest singular value $\sigma_{\min}$ of $\bJ$ rather than on the determinant, and flag any sample below a floor of $10^{-5}$. \textbf{CP} fits Chebyshev interpolants of $\Delta$ of increasing degree and locates their roots \cite{boydComputingZeros2002}. Finally, \textbf{H5 plain} accepts an interval whenever $p_5$ stays positive, omitting the trust band of \cref{eq:trust}.

\begin{figure}[t!]
    \includegraphics[width=1\columnwidth,trim={0.07cm 0.1cm 0.1cm 0cm},clip]{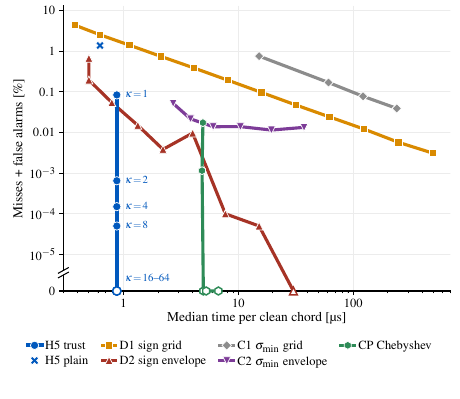}
    \caption{Error rate against computation time for chord certifiers on the KUKA KR6, over $10^6$ random chords that cross a singularity, and $10^6$ that do not. An error is a missed crossing or a refused non-crossing chord. Each curve sweeps the main safety parameter of one method; H5 trust sweeps $\kappa$. Times are medians over non-crossing chords, the cost of accepting a correct answer, measured on C++ implementations of every method. Open markers on the axis indicate no errors observed.}
    \label{fig:certify_front}
\end{figure}

The H5 trust algorithm is the clear winner, outperforming all others in speed and accuracy, with no errors observed for $\kappa\geq16$ and a median time of only \qty{0.88}{\micro\second}. Of the other methods, only CP and D2 also reach zero errors, at $5.7\times$ and $35\times$ the cost. C2 levels off near 0.01\%, and C1 approaches the same level, since their floor also refuses chords that merely pass close to a singularity. D1 still misses 62 crossings per million at $N = 4096$, as a uniform grid cannot resolve arbitrarily narrow excursions. The trust band makes the H5 interpolant viable: without it, H5 plain misses 2.8\% of crossings. Raising $\kappa$ improves reliability at almost no cost: at $\kappa = 16$, 98.7\% of non-crossing chords are cleared by a single interior evaluation and none needs more than nine, and raising $\kappa$ from 1 to 64 increases the mean number of evaluations on such chords only from 1.00 to 1.11. A larger $\kappa$ thus increases the work on only the few chords that need it, and as the error in \cref{eq:hermite_remainder} decreases by approximately $2^6$ at every subdivision, the number of iterations remains largely unchanged. 

\subsection{Aspect-certified restarts}\label{sec:restarts}

Certification can turn a wrong answer into a refusal, but it cannot produce a correct answer by itself. A refused or unconverged attempt may instead be restarted from a new seed $\tilde{\bq}$. The idea of restarts in numerical inverse kinematics is not new, typically implemented via random reseeding~\cite{beesonTRACIK2015}. However, by leveraging the efficient aspect certification of the previous section, we propose \textit{aspect-certified restarts}, in which each new seed is itself certified to lie in the original seed's aspect, and the final solution is certified in the same way.

To form a new seed, we step out from $\bq_0$ along a ray
\begin{align}\label{eq:restart}
    \tilde{\bq}_k &= \bq_0 + s_k \, \b{d}, &
    s_k &= s_1 \, \gamma^{\,k-1} \big( \tfrac{1}{n}\|\b{d}\|_1 \big)^{-1},
\end{align}
where $\b{d}$ is the unit direction of the ray and $k$ is the attempt count. The normalization follows \cref{eq:sampling}, such that $s_1 \gamma^{k-1}$ is the mean absolute joint perturbation, which starts at $s_1$ and grows by a factor $\gamma$ with each attempt. Each new segment $\tilde{\bq}_{k-1} \to \tilde{\bq}_k$, with $\tilde{\bq}_0 = \bq_0$, is certified with the method of \cref{sec:certify}; since the segments are collinear, together they certify the whole ray from $\bq_0$ to $\tilde{\bq}_k$. In this paper, the first ray follows the steepest descent on the error at $\bq_0$, $\b{d} \propto -\bJ^\transpose \b{e}$, which is inexpensive and was found to outperform a random seed direction. A ray that fails certification is discarded and replaced by a random direction.

\begin{figure}[b!]
	\includegraphics[width=1\columnwidth,trim={1cm 0.1cm 1.2cm 0cm},clip]{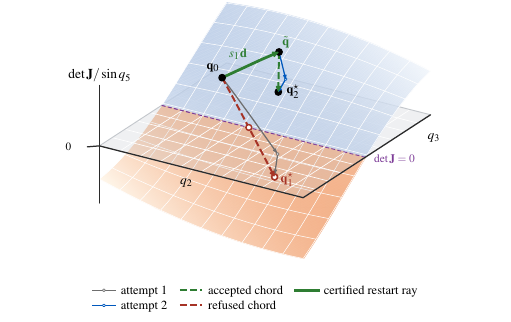}
    \caption{A certified restart on the KUKA KR6, using the plain QuIK solver with no step limiter. The surface plots $\det \bJ / \sin q_5$ over the $(q_2, q_3)$ plane, isolating the shoulder singularity. The first attempt from $\bq_0$ crosses the singularity and converges to $\bq^\star_1$ in another aspect, which is refused. The restart segment $\bq_0 \to \tilde{\bq}$ is certified, and the second attempt converges to the same-aspect solution $\bq^\star_2$, whose chord $\tilde{\bq}\to\bq_2^\star$ is certified.}
	\label{fig:restart}
\end{figure}

The full procedure runs as follows. The first attempt starts from $\bq_0$. If it converges to $\bq^\star$, the chord $\bq_0 \to \bq^\star$ is certified, and a clean chord returns the answer. Otherwise, whether the solver attempt failed or its certification was refused, a new seed is formed: the next candidate $\tilde{\bq}_k$ on the ray is computed, and the new segment $\tilde{\bq}_{k-1} \to \tilde{\bq}_k$ is certified. If the segment is clean, $\tilde{\bq}_k$ becomes the new seed. Otherwise, the ray is discarded and replaced with a random direction, and a new seed is computed and certified. The new attempt runs from this seed, and the process repeats until an answer is accepted or until a maximum number of restarts is reached.

A subtle case is a correct answer whose chord leaves and re-enters the aspect. The certifier correctly refuses the chord, as it exits the aspect, but the answer is nonetheless correct. Such refusals are handled by the procedure above: a restart from a certified seed $\tilde{\bq}_k$ may converge to the same $\bq^\star$, whose chord from the new seed can certify cleanly. In practice, these false refusals are rare: on the KR6, over $10^6$ solves with plain QuIK at $\sigma = 0.3$, certification refused a correct first answer 2\,144 times (0.22\%). Each false refusal was on a chord that leaves and re-enters the aspect, and with restarts, the solver recovered all but 44.

\Cref{fig:restart} illustrates the procedure on the KR6. The first attempt from $\bq_0$ crosses into another aspect and is refused. The restart to $\tilde{\bq}$ is certified, and the second attempt converges to $\bq^\star_2$; the two certified segments $\bq_0 \to \tilde{\bq}$ and $\tilde{\bq} \to \bq^\star_2$ compose into a certified path from $\bq_0$ to $\bq^\star_2$.

\section{Extension to Redundant Arms}\label{sec:redudant}

With $n>6$ joints, the solutions of $\b{e}(\bq) = \b{0}$ form an $r$-dimensional self-motion manifold, $r = n-6$, and $\bJ \in \R^{6\times n}$ has no determinant. 
Its rank-deficient set has codimension $r+1 \geq 2$, so a generic path misses singularities with probability one, and it does not divide the joint space into aspects. A redundant arm therefore does not have aspects to preserve, and the safeguards of \cref{sec:limiter,sec:aspect_certify} do not apply. The solvers of \cref{sec:solvers} can, however; this section summarizes a method for doing so.

Applied directly to a redundant arm, with $\setminus$ taken as the minimum-norm solution, \cref{eq:taylorstep} converges to a point on the self-motion manifold determined by the seed rather than by any criterion, and is therefore not repeatable. The extra freedom is typically used for a secondary goal, such as staying near a nominal configuration and away from the joint limits, or maintaining a well-conditioned pose. We resolve it through the \textit{extended Jacobian} \cite{baillieulKinematicProgramming1985,kleinNewFormulation1995}. A secondary objective $f(\bq)$ is chosen, whose stationary points on the self-motion manifold satisfy $\nabla\! f = \bJ^\transpose\blambda$ for some $\blambda \in \R^6$. Choose six joints $\mathcal{B}$ with $\bJ_\mathcal{B}$ invertible and let $\mathcal{C}$ be the other $r$; this gives $r$ equations
\begin{align}\label{eq:chart}
    \blambda &= \bJ_\mathcal{B}^\transpose \setminus \nabla\! f_\mathcal{B}, &
    \b{h} &= \nabla\! f_\mathcal{C} - \bJ_\mathcal{C}^\transpose \, \blambda = \b{0},
\end{align}
whose solution does not depend on the choice of $\mathcal{B}$. Their Jacobian is
\begin{align}\label{eq:dh}
    \tfrac{\partial \b{h}}{\partial \bq} &= \b{\Phi}_\mathcal{C} - \bJ_\mathcal{C}^\transpose \, \big(\bJ_\mathcal{B}^\transpose \setminus \b{\Phi}_\mathcal{B}\big), &
    \b{\Phi} &= \nabla^2 \! f - \blambda^\transpose \H,
\end{align}
where 
$\blambda^\transpose\H$ denotes the contraction of the kinematic Hessian $\H$ along its six rows. Then, stacking $\b{e}$ and $\b{h}$ gives a square augmented system
\begin{align}\label{eq:augmented}
    \b{e}_{\rm aug} &= \begin{bmatrix} \b{e} \\ \b{h} \end{bmatrix}, &
    \bJ_{\rm aug} &= \begin{bmatrix} \bJ \\ \partial\b{h}/\partial\bq \end{bmatrix} \in \R^{n\times n}.
\end{align}

Expanding all rows to order $p$ to achieve $p+1$ convergence in the higher-order methods of \cref{sec:solvers} would require the derivatives of $\b{\Phi}$ along the step, which in turn requires $\Dm{p+1}$ to be computed explicitly with a free joint index. Since the pose is the primary task and $\b{h}$ only selects among valid solutions, we expand the pose rows alone, and solve the system
\begin{align}\label{eq:redundant_step}
    \dq_{(p)} &= -\begin{bmatrix} \Aop{p} \\ \partial\b{h}/\partial\bq \end{bmatrix} \setminus \begin{bmatrix} \b{e} + \b{c}_p \\ \b{h} \end{bmatrix},
\end{align}
with $\Aop{p}$ and $\b{c}_p$ as in \cref{eq:taylorstep,eq:quik_corr}. The pose error then retains order-$(p+1)$ convergence, and $\b{h}$ converges quadratically. 

We use two secondary objective functions:
\begin{enumerate}
    \item \textit{Objective I} keeps the arm near a nominal configuration $\b{q}_{\rm nom}$ chosen to be well-conditioned, or far from the joint limits. Then,
    \begin{align}
        f = \tfrac12(\bq-\bq_{\rm nom})^\transpose\b{W}(\bq-\bq_{\rm nom}),
    \end{align}
    with trivial derivatives
    \begin{align}
        \nabla\! f &= \b{W}(\bq-\bq_{\rm nom}), & 
        \nabla^2\! f &= \b{W},
    \end{align}
    where $\b{W}$ is a symmetric positive definite weighting matrix that typically scales each joint by its range. With $\bq_{\rm nom}$ at mid-range and a diagonal $\b{W}$, this is the classical joint-limit criterion of \cite{liegeoisAutomaticSupervisory1977}.

    \item \textit{Objective II} maximizes the manipulability \mbox{$w = \sqrt{\det\bJ\bJ^\transpose}$} \cite{yoshikawaManipulability1985}. We minimize \mbox{$f = -\log w$}, which has the same stationary points but better-behaved derivatives
    \begin{align}
        \nabla\! f &= -\frac{\nabla w}{w}, &
        \nabla^2\! f &= -\frac{\nabla^2 w}{w} + \frac{\nabla w\,\nabla w^\transpose}{w^2}.
    \end{align}
    These quantities follow from the closed-form gradient and Hessian of $w$ given in \cite{kimAnalyticalComputation2025}, which require only $\bJ$, $\H$, and $\Dm{3}$ at the current iterate.
\end{enumerate}

A root of the system is a stationary point of $f$, which can be a maximum, minimum, or saddle. From an arbitrary seed, iteration with \cref{eq:redundant_step} will often land on a saddle. It is therefore standard \cite{kleinNewFormulation1995} to warm-start the method by first performing first-order null-space descent with step size $\mu$ \cite{liegeoisAutomaticSupervisory1977}
\begin{align}\label{eq:liegeois}
    \dq &= -\bJ^+\b{e} - \mu\,(\b{I}-\bJ^+\bJ)\,\nabla\! f, &
    \bJ^+ &= \bJ^\transpose(\bJ\bJ^\transpose)^{-1}.
\end{align}
This requires no second derivatives and converges reliably to a minimum rather than a saddle, but more slowly and at higher cost. Subsequent solves along a trajectory are already in the basin, and do not require this initial step.

The full expansion is possible but costly, with limited practical benefit. \Cref{tab:redundantcost} compares the three strategies per iteration for 7-DOF and 16-DOF chains, and per solve on the 7-DOF KUKA iiwa at $\sigma = 0.01$\,rad, reproducing one step of trajectory tracking. The second-order expansion preserves the higher-order convergence of $\b{e}$ at a modest cost: \qty{0.66}{\micro\second}/iteration for QuIK at $n=7$ against \qty{0.34}{\micro\second} pose-only, and \qty{1.9}{\micro\second} per tracking solve, well within a real-time budget. The full-order expansion matches the convergence order of $\b{h}$ to that of $\b{e}$ and saves up to one iteration, but the per-iteration cost grows steeply with both $p$ and $n$, and the total solve time is worse in every case. Null-space descent converges linearly in $\b{h}$ and needs 27 iterations and significantly more time, but is necessary to avoid saddle points when no prior solution is available. The library implements the second-order expansion for every step method, and can warm-start the first solve of a trajectory with null-space descent.

\begin{table}[t!]
	\caption{Cost of resolving the redundancy: one iteration, and one tracking step.}
	\label{tab:redundantcost}
	\footnotesize
	\renewcommand{\arraystretch}{1.3}
	\setlength{\tabcolsep}{3pt}
	\begin{tabularx}{\columnwidth}{
		@{}c@{\hspace{1.5em}}
		X@{\hspace{1em}}
		r@{\hspace{1em}}
        r@{\hspace{2.5em}}
        r@{\hspace{1em}}
        r@{\hspace{3em}}
        r@{\hspace{1.75em}}
        r@{}
	}
		& & \multicolumn{2}{r@{\hspace{2.5em}}}{\textbf{Order}} &
            \multicolumn{2}{@{}r@{\hspace{3em}}}{\textbf{Iteration [\si{\micro\second}]}} & 
            \multicolumn{2}{r@{}}{\makebox[0pt][r]{\textbf{Solve ($\sigma\!=\!0.01$)}}} \\
		& & $\b{e}$ & $\b{h}$ & $n=7$ & $n=16$ & Iter. & [\si{\micro\second}] \\
		\toprule
		\input figures/tab_redundant_cost_lib_v2
		\bottomrule
	\end{tabularx}
	\\[1em]
	\footnotesize Note: Objective~I ($\b{W} = \b{I}$). \emph{Order}: measured convergence order of the pose error $\b{e}$ and of the optimality condition $\b{h}$ on the 7-DOF KUKA iiwa; the full-order block is a prototype, not in the library. \emph{Iteration}: the time of one iteration. \emph{Solve}: mean iterations and time of a solve at $n=7$ at $\sigma=0.01$, reproducing one trajectory-tracking step.
\end{table}

\section{Experimental Validation}\label{sec:validation}

The solvers and safeguards of \cref{sec:solvers,sec:limiter,sec:aspect_certify} are implemented in C++ with Eigen~3.4 and released as a library with ROS~2 integration. The implemented solvers are highly optimized, leveraging the efficient derivative evaluations from \cref{sec:derivs} and caching per-step quantities (kinematics, factorizations, and derivative contractions) for reuse across the solver, limiter, and certifier. The nominal QuIK method is implemented as in \cite{lloydFastRobust2022}, but also integrates a tunable joint-step limit that bounds the step norm to $\theta_{\rm max}$ as in \cref{eq:alpha_model}.
 
In this section, we evaluate the assembled solver on random problems, sampled as in \cref{eq:sampling}, under a hard deadline~$\tau$, scored by failure rate: the fraction of solves that do not return a same-aspect answer within~$\tau$. These failures can be further subdivided into \textit{refusals}, where the solver is unable to converge or refuses to answer, and \textit{wrong-aspect answers}, where the solver returns an answer in a different aspect as if correct. A refusal is visible to the caller, which can hold its previous setpoint safely; a wrong-aspect answer is a valid-looking configuration whose execution drives the arm through a singularity, at unbounded joint speed in the worst case, and is a far more dangerous outcome.
 
Testing is performed at two representative error levels: \mbox{$\sigma=1$\,rad}, representing a cold-start scenario, and $\sigma=0.01$\,rad, representing a trajectory-tracking step. Six deadlines are tested, from \qty{1}{\milli\second}, representing an appropriate share of the control period in 100--500\,Hz loops, down to \qty{10}{\micro\second}, representing servo rates of \num{1}--\qty{10}{\kilo\hertz} or higher, or lower-performance hardware at moderate rates. The KUKA KR6 is used throughout unless otherwise noted.

\subsection{Experimental setup}\label{sec:setup}

In practice, an inverse kinematics solver can be tuned for a given chain, deadline, or error level. At more permissive deadlines, a slower, more cautious convergence may be preferable, whereas at tighter deadlines, convergence must be fast and efficient. To this end, a tuning procedure is designed and packaged into the codebase, allowing users to fine-tune the methods to their situation. Accordingly, each data point in this study is tuned for its chain, method, deadline, and error level, such that each reported failure rate is the best found by the tuning for that configuration. Tuning fits a cost model
\begin{align}
    \hat{t} = \lambda_0 + \lambda_1 \, N_{\rm iter} + \lambda_2 \, N_{\rm probes} + \lambda_3 \, N_{\rm restart}
\end{align}
predicting the time of a solve from counted work. The deadline is then applied through the model rather than the CPU clock, such that the objective, the failure rate under the deadline, is deterministic and free of timing noise, and the tuning search can run in parallel on several cores. Parameters are optimized by a coordinate pattern search in two stages, on $2.5\cdot10^4$ and $10^5$ problems, and the best candidates are re-scored on a separate set of $2\cdot10^5$ problems. At $\sigma = 0.01$, these sizes rise to $10^6$, $2\cdot10^6$, and $2\cdot10^6$ to resolve the lower failure rates. Even so, near one failure in $10^6$, the selection set holds too few failures to rank candidates reliably, and better parameters may exist in that regime.

Every in-library solver tunes the joint-step cap $\theta_{\rm max}$ ($0.05$--$4$\,rad), a Cartesian clamp scale as in \cite{lloydFastRobust2022}, relative-improvement and gradient-failure tolerances, and the QuIK4/QuIK5 substituted-step margin $c$. The step limiter adds $\beta_{\rm fm}$, $\beta_{\rm bt}$, the probe budget, and the $\alpha$-floor $\alpha_{\rm min}$. Certified restarts add the retry count ($1$--$16$), initial perturbation size $s_1$, growth factor $\gamma$, and rays certified per retry. The following five prior methods are tested alongside the proposed safeguards:

\begin{itemize}
    \item \textbf{Selectively damped least squares (SDLS)}~\cite{bussSelectivelyDamped2005}, which damps each singular direction of the step separately; tuned parameter: $\gamma_{\rm max}$.
    \item \textbf{Adaptive damping}~\cite{nakamuraInverseKinematic1986,chiaveriniReviewDamped1994}, which raises the damping factor only when the smallest singular value of $\bJ$ drops below a threshold; tuned parameter: $\sigma_\varepsilon$.
    \item \textbf{DMH}~\cite{chenFifthOrderPOE2026}, the fifth-order damped modified Halley method, which follows a Halley step with a Newton correction using a frozen Hessian; reimplemented on this library's DH kinematics, as its released product-of-exponentials code was significantly slower, leading to an unfair comparison. Tuned parameters: damping and maximum iterations.
    \item \textbf{KDL}~\cite{smitsKDLKinematics,bruyninckxOpenRobot2001}, the Levenberg--Marquardt solver of the Orocos library, run as released and selected from maximum iterations of 20, 100, or 500.
    \item \textbf{TRAC-IK}~\cite{beesonTRACIK2015}, which runs a Newton--Raphson solver in parallel with an SQP formulation and returns the first answer; run as released but pinned to one core, selected from Speed or Distance mode.
\end{itemize}
SDLS and adaptive damping are proposed as damping schemes only; accordingly, they are implemented inside this library on each of NR, QuIK, QuIK4, and QuIK5, and the best at each point is plotted. Problem generation follows \cref{eq:sampling}, except where otherwise specified.

All timings are single-threaded C++ compiled with {\small\texttt{-O3}} optimizations, measured on a pinned core of a MacBook Pro with M5 Pro chipset. Each data point begins with at least $N = 10^6$ problems, and the sample size is then increased until 100 failures are observed, which gives a 95\% confidence interval of approximately $\pm 20\%$, or until $N = 10^9$. Since the test machine is not a real-time system, each failure at the deadline boundary is re-measured five times and counted only if it fails in the majority.

\subsection{Evaluation of the safeguards}\label{sec:safeguard_eval}

\Cref{fig:ladder} tests each safeguard level on the KR6 at \mbox{$\sigma = 1$} and $0.01$\,rad, with parameters tuned per deadline. Without safeguards, the failure rate levels off at 2.0\% for \mbox{$\sigma = 1$}, and at 0.018\% for \mbox{$\sigma = 0.01$}; the majority of these failures are wrong-aspect answers, which are not avoided by faster convergence.

The step limiter provides increasing benefit as the deadline grows. At $\tau = \qty{1}{\milli\second}$ and $\sigma = 1$, it reduces failures $8\times$ to 0.26\%; at $\sigma = 0.01$, it provides a $290\times$ improvement to $6.3\cdot10^{-5}$\%. Further adding certified restarts extends the gain where the time budget permits multiple attempts: at $\sigma = 1$ they reach 0.031\% ($65\times$ below nominal), and at $\sigma = 0.01$ fewer than one failure per $10^8$ problems from $\tau = \qty{125}{\micro\second}$ onward. At the tightest deadline, $\tau = \qty{10}{\micro\second}$, and $\sigma = 1$, there is not enough time for certification or extra attempts, so limiter+restarts fail more than the limiter alone; for $\tau\geq\qty{25}{\micro\second}$, they overtake it.

The step method (QuIK, QuIK4, or QuIK5) makes little difference without a safeguard, as higher-order steps do not prevent aspect switches. With the limiter, higher-order methods perform better: for $\tau=\qty{1}{\milli\second}$, QuIK4 fails $7.3\times$ less often than QuIK at $\sigma=1$, and $18\times$ less often at $\sigma=0.01$, as its faster convergence offsets the limiter's per-step overhead. QuIK5 shows similar gains, but with diminishing returns at its higher per-iteration cost.

\begin{figure}[t!]
    \includegraphics[width=1\columnwidth,trim={0.47cm 0.3cm 0.05cm 0.2cm},clip]{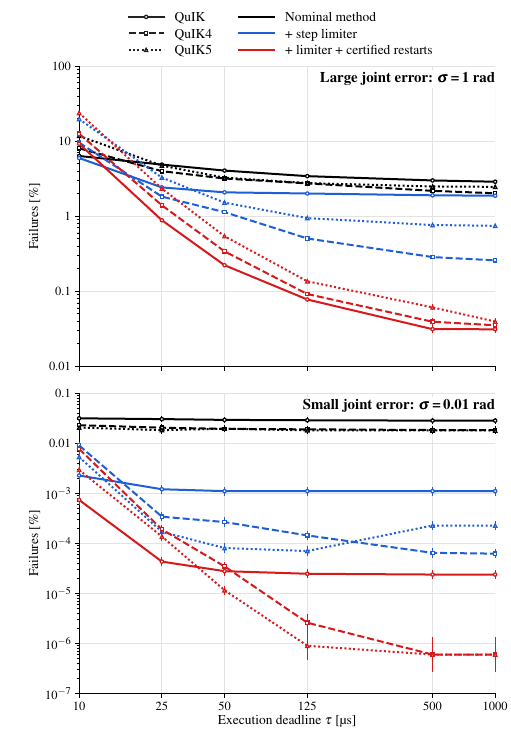}
    \caption{Failure rate on the KUKA KR6 against the solver deadline $\tau$, at large ($\sigma = 1$~rad, top) and small (\mbox{$\sigma = 0.01$}~rad, bottom) initial joint error. A failure is any solve that does not return a same-aspect answer within $\tau$. Each point uses the parameters tuned for its deadline and error level.}
    \label{fig:ladder}
\end{figure}

\subsection{Comparison with prior solvers}\label{sec:comparison}

\begin{figure}[p!]
    \includegraphics[width=1\columnwidth,trim={0.25cm 0.2cm 0.05cm 0.1cm},clip]{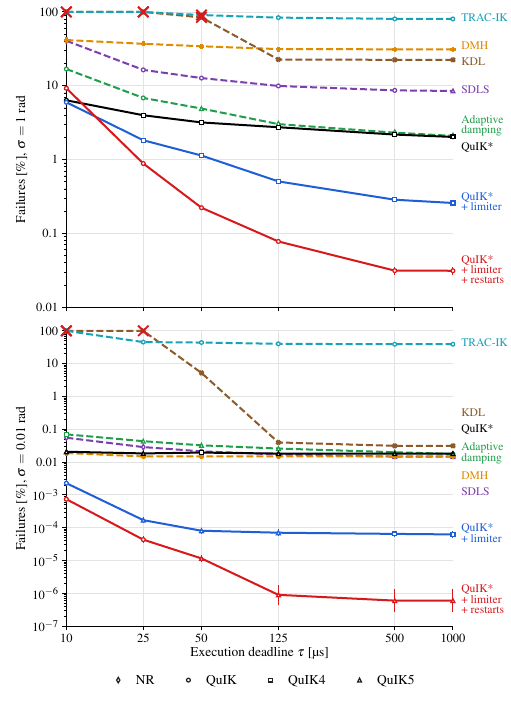}
    \caption{Failure rate on the KUKA KR6 against the solver deadline $\tau$ at $\sigma = 1$~rad (top) and $\sigma = 0.01$~rad (bottom), compared with prior methods. Solid lines are the safeguard levels of \cref{fig:ladder}, each at the best of QuIK, QuIK4, and QuIK5 at that point (denoted as \textit{QuIK*}); the marker shows which. Prior methods (dashed) are tuned per deadline and $\sigma$ for their own lowest failure rate; SDLS and adaptive damping run on each of Newton--Raphson (NR), QuIK, QuIK4, and QuIK5, and the marker shows the best. A red cross marks a deadline shorter than the method's median solve time.}
    \label{fig:sota}
\end{figure}

\Cref{fig:sota} compares the proposed safeguard levels against the prior solvers of \cref{sec:setup}. Solid lines plot the best of QuIK, QuIK4, and QuIK5 at each level; dashed lines show the baselines, each tuned to its own lowest failure rate.

At $\sigma = 1$, the prior methods fail far more often. KDL and TRAC-IK could not run reliably below $\tau = \qty{125}{\micro\second}$; even at longer deadlines, TRAC-IK failed on 80\% of problems and KDL on 23\%. DMH leveled off at 31\%, limited by convergence to the wrong solution basin rather than by speed. SDLS and adaptive damping reached 8.5\% and 2.1\%, but both performed worse on average than nominal (safeguard-free) QuIK: the singular-value decomposition required at each step is expensive, and both can stall the solver near ill-conditioned configurations, even when the step points away from the singularity.
At $\tau = \qty{1}{\milli\second}$, the limiter and limiter+restarts reach 0.26\% and 0.031\%, an $8\times$ and $67\times$ improvement over the best prior method (adaptive damping, 2.1\%). 

At $\sigma = 0.01$, aspect switching is rarer and deadlines are easier to meet, so every method fails less often. The prior methods level off at 0.015--0.03\% (TRAC-IK excepted), nearly all wrong answers, while the safeguards continue to improve, so the relative gap widens. The limiter reaches $6.3\cdot10^{-5}$\% ($240\times$ below SDLS, the best prior method), and with limiter+restarts only 6 failures were observed in $10^9$ samples ($6\cdot10^{-7}$\%): a $25\,000\times$ improvement.

\begin{table}[p!]
	\caption{Wrong answers and refusals at $\sigma = 1$~rad, $\tau = \qty{1}{\milli\second}$.}
	\label{tab:outcomes}
	\footnotesize
	\renewcommand{\arraystretch}{1.3}
	\setlength{\tabcolsep}{3pt}
	\sisetup{mode=text, detect-weight=true}
	\begin{tabularx}{\columnwidth}{
		@{}X
		S[table-format=2.3]
		S[table-format=2.6]@{\hspace{1.5em}}
		S[table-format=3.2]
		S[table-format=1.3]@{}
	}
		& & & \multicolumn{2}{c@{}}{\textbf{Refusals [\%]}} \\
		& {\textbf{Failed [\%]}} & {\textbf{Wrong [\%]}} & {Correct} & {Incorrect} \\
		\toprule
		\input figures/tab_compact_v4
		\bottomrule
	\end{tabularx}
	\\[.5em]
    \footnotesize Note: \textit{Failed} is any problem that did not converge to a same-aspect answer, measured on at least $10^6$ problems per row and extended until 100 failures are observed. Of these, \textit{Wrong} is the percentage answered in the wrong aspect, and \textit{Incorrect refusals} the percentage refused when a correct answer existed. \textit{Correct refusals} are measured on a separate $N=10^5$ set whose seed aspect holds no solution. The label \textit{QuIK*} is used to denote the best of QuIK, QuIK4, or QuIK5 after tuning. SDLS and adaptive damping also use the best of NR, QuIK, QuIK4, or QuIK5. The best entry in each column is bold.

    $\dagger$ No wrong answer in $10^6$ problems (95\% upper bound 0.00027\%).
\end{table}

\Cref{tab:outcomes} separates the failure types at $\tau = \qty{1}{\milli\second}$. The distinction that matters most is between wrong answers, where the solver returns a configuration in another aspect, and refusals, where it declines to answer. A refusal is safer: visible to the caller, which can hold its previous setpoint; a wrong answer is not. With the nominal solver, 1.35\% of solvable problems produce a wrong-aspect answer. The limiter reduces this $28\times$ to 0.05\%, and with certified restarts no wrong answer was returned in $10^6$ problems, while nearly all unsolvable problems were correctly refused. Prior methods show the opposite trend: the majority of failures are wrong answers, from 1.3\% with adaptive damping to 79\% with TRAC-IK, and they return wrong answers on unsolvable problems at much higher rates: from 16\% for SDLS to 97\% for TRAC-IK, with only adaptive damping (0.5\%) performing reasonably.

\subsection{Generality and cuspidal chains}\label{sec:generality}

\Cref{fig:generality} repeats the measurement with QuIK4 on three further chains: the Stanford arm~\cite{scheinmanDesign1969}, a 6-DOF chain with one prismatic joint; the Denso VS-6556, a Puma-style arm rendered cuspidal by its kinematic calibration~\cite{dehghaniRobotMounted2023,lloydPrecisionRobotic2024}; and the Kinova Jaco~\cite{Campeau-Lecours2017}, cuspidal by design through its non-spherical wrist~\cite{salunkheTrajectory2023}. Here, not all chains have closed-form solutions to determine their aspects, so a problem is kept only when the straight chord from $\bq_0$ to $\bq^\star$ is certified free of singularities; this is a stricter filter than used in \cref{fig:ladder,fig:sota}. On the cuspidal chains, several solutions can share the seed's aspect, so a solve succeeds when it returns $\bq^\star$ or another solution proven to lie in its aspect, by a certified chord from $\bq_0$ or $\bq^\star$ or along the solver's own path; an answer that cannot be decided counts as a failure.
To avoid grading the safeguards by their own certifier, the dataset labeling used a more expensive offline labeler: a Chebyshev interpolant of $\Delta$ of degree up to 64 with global root finding~\cite{boydComputingZeros2002}, along with the certifier of \cref{sec:certify} with a very large safety factor, $\kappa = 1024$, run separately on 64 uniform sub-intervals such that its sample points never coincide with the runtime certifier. A chord is considered clean when the two mechanisms certify it. Against the KR6's closed-form aspect label, this approach made no error on $10^6$ random chords at $\sigma = 1$.

\begin{figure}[b!]
    \includegraphics[width=1\columnwidth,trim={0.25cm 0.2cm 0.1cm 0.1cm},clip]{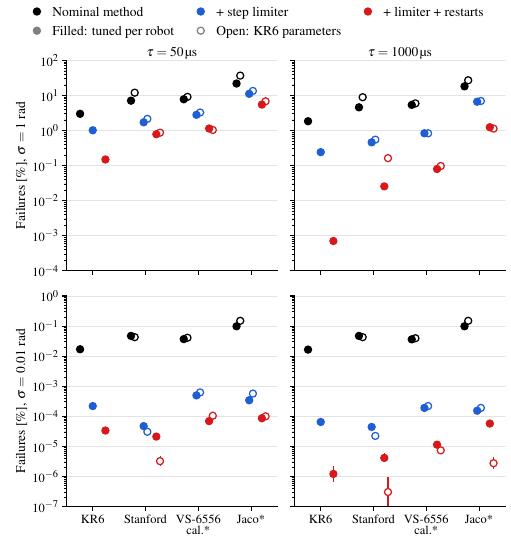}
    \caption{Failure rate of the safeguards on four chains with QuIK4, at large ($\sigma = 1$~rad, top) and small ($\sigma = 0.01$~rad, bottom) initial joint error, for deadlines \mbox{$\tau = \qty{50}{\micro\second}$} and \qty{1}{\milli\second}. Filled markers use parameters tuned for each robot; open markers reuse the KR6's parameters. Problems are kept only when the chord $\bq_0 \to \bq^\star$ is certified singularity-free. On the non-cuspidal \textit{KR6} and \textit{Stanford}~\cite{scheinmanDesign1969} chains, a failure is any solve that does not return $\bq^\star$ within $\tau$; on the cuspidal (starred) \textit{VS-6556 cal.} and \textit{Jaco}~\cite{Campeau-Lecours2017} chains, any solve that does not return $\bq^\star$ or another solution proven to lie in its aspect, with an undecided answer counted as a failure. Bars are 95\% Wilson intervals; a triangle marks a point with no failures, drawn at its 95\% upper bound.}
    \label{fig:generality}
\end{figure}

The safeguards largely carry over to the other chains. On the Stanford arm at $\sigma = 1$ and $\tau = \qty{1}{\milli\second}$, the limiter reduces failures by $10\times$ and limiter+restarts by $180\times$. On the cuspidal chains, failures on the VS-6556 fall from 5.4\% to 0.84\% with the limiter and 0.080\% with limiter+restarts, and on the Jaco from 18\% to 6.6\% and 1.2\% respectively. At $\sigma = 0.01$, the limiter alone reduces failures $190\times$ on the VS-6556 and $640\times$ on the Jaco, to $1.9\cdot10^{-4}$\% and $1.5\cdot10^{-4}$\%, and limiter+restarts reach $1.1\cdot10^{-5}$\% and $6\cdot10^{-5}$\%. Without safeguards, 4.0\% (VS-6556) and 9.9\% (Jaco) of solves at $\sigma = 1$ return an answer proven to lie in another aspect; with either the limiter or limiter+restarts enabled, no wrong answers were observed at any deadline.

Parameters also transfer: the KR6's parameters are no more than a factor of two worse than each chain's own tuning at nearly every point, the exception being the Stanford arm's restart cell at $\sigma = 1$ and $\tau = \qty{1}{\milli\second}$. At $\sigma = 0.01$ the transferred parameters sometimes differ, as failure rates there lie below the resolution of the tuning sets (\cref{sec:setup}).

Overall, failure rates are higher on the cuspidal arms. This reflects harder problems rather than a loss of aspect safety: without safeguards, QuIK4 already fails $3\times$ (VS-6556) and $10\times$ (Jaco) as often as on the KR6, and with the safeguards none of the remaining failures is a proven aspect change.

The calibrated \textit{Denso VS-6556} illustrates a key case. Industrial arms typically have aligned joints, closed-form inverse kinematics, and non-cuspidal geometry. Kinematic calibration, necessary for maximum accuracy~\cite{dehghaniRobotMounted2023}, perturbs these parameters, invalidating the closed-form solution and requiring a numerical solver. This introduces risk in real-time applications: a divergent solver could halt operations or, in the worst case, produce unbounded joint speeds. With the methods of this paper, failure rates on the calibrated arm at trajectory-tracking error levels ($\sigma=0.01$) fall to about 1 in~$10^7$, all of them refusals on near-singular targets rather than wrong answers. As real-world operations typically avoid singular workspace regions, numerical inverse kinematics with the aspect-preserving methods introduced here can be used in real-time control with much greater safety than before.

\section{Conclusion}\label{sec:conclusion}

This paper used the higher-order kinematic derivatives of a serial chain to serve three complementary purposes in real-time inverse kinematics. First, Taylor-model solvers of third and fourth order were derived, along with the orientation-error correction map needed beyond second order, yielding the QuIK4 and QuIK5 methods with fourth- and fifth-order convergence respectively. Second, an aspect-preserving step limiter was introduced that efficiently extrapolates $\det\bJ$ to predict singularities along the proposed step, and shortens the step as necessary. Unlike damped least-squares methods, which shorten or warp the step based on the conditioning of $\bJ$ at the current iterate, the limiter acts on a prediction of where the step is heading. This allows it to limit steps surgically: a step from a well-conditioned pose toward a singularity is still shortened, while one at an ill-conditioned pose facing away from the singularity is left intact, achieving high aspect safety without penalizing convergence speed. Lastly, an aspect-certification algorithm, built from Hermite interpolation of the determinant along the chord from seed to answer, verifies each solution, and \textit{certified restarts} recover refused or failed solves from seed points proven to lie in the same aspect.

In fixed-deadline testing, the limiter alone reduced failures by up to $8\times$ at large seed error ($\sigma=1$\,rad) and $240\times$ at tracking-level error ($\sigma=0.01$\,rad) over prior methods. Combined with certified restarts, these improvements grew to $67\times$ and $25\,000\times$, with failure rates as low as $6\cdot 10^{-7}$\%. The safeguards transferred to cuspidal chains, including a calibrated industrial arm and the Kinova Jaco, with no proven aspect changes at any tested deadline or error level.

The solvers, safeguards, and tuning tools are released as a C++ library with ROS~2 integration~\cite{quikpCode}. When a numerical inverse kinematics solver can deliver aspect-preserving answers in a fraction of a millisecond, the longstanding barriers to its use in real-time control are removed, and applications that depend on it, such as kinematically calibrated robots or manipulators lacking a closed-form solution, can operate with a level of safety that was previously out of reach.

\appendix[Order-$(p+1)$ Taylor-Model Convergence]
\label{app:order}

\begin{proposition}\label{prop:order}
    Let $\b{F}$ be $p+1$ times continuously differentiable near a root $\bq^\star$ at which
    $\bJ(\bq^\star)$ is invertible, and let $\Dm{m}$ be the exact derivatives of $\b{F}$. Then
    there is a constant $C$ such that, for every $\bq$ close enough to $\bq^\star$, the step
    $\dq_{(p)}$ of \cref{eq:taylorstep} satisfies
    \begin{align*}
        \big\|\bq + \dq_{(p)} - \bq^\star\big\| \le C\,\big\|\bq - \bq^\star\big\|^{p+1},
    \end{align*}
    so the iteration converges locally with order $p+1$.
\end{proposition}

\begin{IEEEproof}
    Write $\b{A}_p(\b{v}) = \sum_{m=1}^{p}\tfrac{1}{m!}\Dm{m}[\b{v}^{m-1}]$, so that the
    Taylor polynomial of $\b{F}$ about $\bq$ is $\b{F}(\bq) + \b{A}_p(\b{v})\,\b{v}$. Let
    $\bdelta^\star = \bq^\star - \bq$ and $\varepsilon = \|\bdelta^\star\|$. The steps of
    \cref{eq:taylorstep} satisfy $\b{A}_p(\dq_{(p-1)})\,\dq_{(p)} = -\b{F}(\bq)$, with
    $\dq_{(1)} = \dq_{\rm nr}$. Since $\bq + \dq_{(p)} - \bq^\star = \dq_{(p)} - \bdelta^\star$,
    it suffices to show that $\dq_{(p)} - \bdelta^\star = O(\varepsilon^{p+1})$. We proceed by
    induction on $p$.

    For $p=1$ this is Newton's method. Assume the result for $p-1$. Since
    $\b{F}(\bq^\star) = \b{0}$, Taylor's theorem gives
    $\b{A}_p(\bdelta^\star)\,\bdelta^\star = -\b{F}(\bq) + O(\varepsilon^{p+1})$. Subtracting
    the defining equation of $\dq_{(p)}$,
    \begin{align*}
        \b{A}_p(\dq_{(p-1)})\,\big(\dq_{(p)} - \bdelta^\star\big)
        &= \big[\b{A}_p(\bdelta^\star) - \b{A}_p(\dq_{(p-1)})\big]\,\bdelta^\star \\
        &\quad + O(\varepsilon^{p+1}).
    \end{align*}
    The bracket equals $\tfrac12\Dm{2}[\bdelta^\star - \dq_{(p-1)}] + O(\varepsilon^{p+1})$,
    which is $O(\varepsilon^{p})$ by the induction hypothesis, so the right side is
    $O(\varepsilon^{p+1})$. Near a regular root, $\b{A}_p(\dq_{(p-1)}) = \bJ + O(\varepsilon)$
    is invertible with a bounded inverse, and therefore
    \mbox{$\dq_{(p)} - \bdelta^\star = O(\varepsilon^{p+1})$}.
\end{IEEEproof}

The proposition requires $\Dm{m}$ to be the exact derivatives of $\b{F}$. For the position
rows of the error they are; for the orientation rows they are not beyond second order,
which is the mismatch that the corrections $\b{c}_p$ of \cref{eq:quik_corr} remove.

\balance
\bibliographystyle{IEEEtran}
\bibliography{quikp_clean}

\end{document}

%% file: packages.tex
\usepackage[hidelinks]{hyperref}
\usepackage{amsmath}
\usepackage{amssymb}
\usepackage{graphicx}
\usepackage{array}
\usepackage{bm}
\usepackage{upgreek}
\usepackage{cases}
\usepackage[longtable]{multirow}
\usepackage{mathtools}
\usepackage{tabularx}
\usepackage{siunitx}
\usepackage{bbm}
\usepackage{dblfloatfix}
\usepackage{booktabs}
\usepackage{color}
\usepackage{lipsum}
\usepackage{cite}
\usepackage[capitalize,nameinlink]{cleveref}
\usepackage{printlen}
\usepackage{tikz}
\usepackage{comment}

\usepackage{enumitem}
\usepackage{balance}

\usepackage{times} %

%% file: commands.tex
\def\cwidth{8.64468cm}
\def\twidth{17.71106cm}
\newlength{\cwidthLength}
\newlength{\twidthLength}
\newcolumntype{L}{>{\raggedright\arraybackslash}X}
\newcolumntype{C}{>{\centering\arraybackslash}X}
\newcolumntype{R}{>{\raggedleft\arraybackslash}X}

\crefname{equation}{}{}
\crefname{figure}{Fig.}{Figs.}
\crefname{table}{Table}{Tables}
\crefname{section}{Section}{Sections}

\newcommand{\R}{\mathbb{R}}
\renewcommand{\b}[1]{\mathbf{#1}}
\newcommand{\bscript}[1]{\hspace{-0.05em}{}^{#1\hspace{-0.05em}}}
\newcommand{\transpose}{\mathsf{T}}
\DeclareMathOperator{\tr}{tr}

\newcommand{\bq}{\b{q}}

\DeclareMathAlphabet{\mathsfb}{T1}{phv}{b}{n}
\newcommand{\tensor}[1]{\mathsfb{#1}}

\renewcommand{\H}{\tensor{H}}       %

\NewDocumentCommand{\bT}{ o o }{
	\IfNoValueF{#1}{\bscript{#1}}
	\IfNoValueTF{#2}{\b{T}}{\b{T}_{#2}}
}
\NewDocumentCommand{\bJ}{ s o o }{
	\IfNoValueF{#2}{\bscript{#2}}
	\IfBooleanTF {#1}
		{\IfNoValueTF{#3}{\dot{\b{J}}}{\dot{\b{J}}_{#3}}}
		{\IfNoValueTF{#3}{\b{J}}{\b{J}_{#3}}}
}
\NewDocumentCommand{\bR}{ s o o }{
	\IfNoValueF{#2}{\bscript{#2}}
	\IfBooleanTF {#1}
		{\IfNoValueTF{#3}{\dot{\b{R}}}{\dot{\b{R}}_{\rm #3}}}
		{\IfNoValueTF{#3}{\b{R}}{\b{R}_{#3}}}
}

\newcommand{\etal}{\textit{et al.}}

\newcommand{\Dm}[1]{\tensor{D}^{(#1)}}        %
\newcommand{\dq}{\delta\bq}                    %
\newcommand{\dqp}[1]{\!\big[\dq^{\,#1}\big]}   %
\newcommand{\Aop}[1]{\b{A}_{#1}}               %

\def\sizeA{0.18cm}

\makeatletter
\newcommand*\circled[1]{%
  ~\smash{\tikz[baseline={([yshift=-0.1em]char.base)}]{%
    \draw[fill=white] (0,0) circle (\sizeA);%
    \node[inner sep=1pt,font=\scriptsize,align=center] (char) at (0,0) {#1};%
  }}%
  \futurelet\@let@token\circled@sp}
\def\circled@sp{%
  \ifx\@let@token\@sptoken\else\kern0.1em\fi}
\makeatother

\DeclareSymbolFont{ugrbold}{U}{eur}{b}{n}
\DeclareMathSymbol{\bupalpha}{\mathord}{ugrbold}{"0B}
\DeclareMathSymbol{\bupomega}{\mathord}{ugrbold}{"21}
\DeclareMathSymbol{\buplambda}{\mathord}{ugrbold}{"15}
\DeclareMathSymbol{\buppsi}{\mathord}{ugrbold}{"20}
\DeclareMathSymbol{\bupiota}{\mathord}{ugrbold}{"13} 
\DeclareMathSymbol{\bupdelta}{\mathord}{ugrbold}{"0E} \newcommand{\bdelta}{\bupdelta}
\newcommand{\bpsi}{\buppsi}
\newcommand{\bomega}{\bupomega}
\newcommand{\blambda}{\buplambda}

\newcommand{\skewMat}[1]{\begin{bmatrix}{#1}\end{bmatrix}_{\!\scalebox{1}{\ensuremath{\times}}}}

\definecolor{darkaqua}{RGB}{0,114,189}

\renewcommand{\th}{\textsuperscript{th}}

%% file: figures/tab_derivcost.tex
$\bT$, $\bJ$ (FK/Jacobian) & 378 / 526 & 0.057 && 378 / 526 & 0.057 \\
$\Dm{2}\dqp{} = \H\dqp{}$ & 324 / 432 & 0.037 && 162 / 144 & 0.021 \\
$\Dm{3}\dqp{2}$ & 1\,569 / 1\,878 & 0.298 && 270 / 252 & 0.039 \\
$\Dm{4}\dqp{3}$ & 8\,364 / 9\,204 & 2.757 && 378 / 396 & 0.046 \\
\midrule
\textbf{Total (all orders)} & 10\,635 / 12\,040 & 3.149 && 1\,188 / 1\,318 & 0.163 \\

%% file: figures/tab_redundant_cost_lib_v2.tex
\multirow{4}{*}{\rotatebox[origin=c]{90}{\scriptsize\bfseries\shortstack{Pose only}}} & NR & 2.00 & {--} & 0.20 & 0.38 & 3.1 & 0.7 \\
 & QuIK & 3.00 & {--} & 0.34 & 0.57 & 2.0 & 0.8 \\
 & QuIK4 & 4.00 & {--} & 0.53 & 0.93 & 2.0 & 1.1 \\
 & QuIK5 & 5.00 & {--} & 0.81 & 1.47 & 1.9 & 1.6 \\
\midrule
\multirow{4}{*}{\rotatebox[origin=c]{90}{\scriptsize\bfseries\shortstack{+ redundancy\\[-1.5pt](2nd-order)}}} & NR & 2.00 & 1.99 & 0.53 & 1.56 & 3.2 & 2.1 \\
 & QuIK & 3.00 & 1.99 & 0.66 & 2.26 & 2.2 & 1.9 \\
 & QuIK4 & 4.00 & 1.99 & 0.86 & 3.13 & 2.7 & 2.8 \\
 & QuIK5 & 4.99 & 1.99 & 1.14 & 4.21 & 2.9 & 3.7 \\
\midrule
\multirow{4}{*}{\rotatebox[origin=c]{90}{\scriptsize\bfseries\shortstack{+ redundancy\\[-1.5pt](full-order)}}} & NR & 2.00 & 1.99 & 0.49 & 1.51 & 3.1 & 3.3 \\
 & QuIK & 3.00 & 2.99 & 1.18 & 4.46 & 2.1 & 3.7 \\
 & QuIK4 & 4.00 & 4.00 & 2.57 & 10.34 & 2.0 & 6.4 \\
 & QuIK5 & 5.00 & 4.97 & 4.85 & 21.03 & 2.0 & 10.5 \\
\midrule
\multicolumn{2}{@{}l}{Null-space descent} & 2.00 & 1.02 & 0.21 & 0.36 & 27.4 & 72.0 \\

%% file: figures/tab_compact_v4.tex
Nominal QuIK* method & 2.03 & 1.35 & 99.1 & 0.67 \\
\hspace{0.5em}+ limiter & 0.26 & 0.049 & 99.6 & 0.21 \\
\hspace{0.5em}+ limiter + restarts & \bfseries 0.031 & {\bfseries (none)$^\dagger$} & \bfseries 100.0 & \bfseries 0.031 \\
\midrule
SDLS & 8.51 & 8.17 & 84.2 & 0.34 \\
Adaptive damping & 2.09 & 1.33 & 99.5 & 0.76 \\
DMH & 31.2 & 22.0 & 66.6 & 9.22 \\
KDL & 22.5 & 21.7 & 36.3 & 0.75 \\
TRAC-IK & 80.4 & 79.3 & 3.38 & 1.11 \\